\documentclass[journal]{IEEEtran}
\usepackage[pagebackref,breaklinks,colorlinks]{hyperref}
\usepackage{url}
\usepackage[utf8]{inputenc} % allow utf-8 input
\usepackage[T1]{fontenc}    % use 8-bit T1 fonts
\usepackage{hyperref}       % hyperlinks
\usepackage{url}            % simple URL typesetting
\usepackage{booktabs}       % professional-quality tables
\usepackage{amsfonts}       % blackboard math symbols
\usepackage{nicefrac}       % compact symbols for 1/2, etc.
\usepackage{microtype}      % microtypography
\usepackage{xcolor}         % colors
\usepackage{amsmath}
\usepackage{multicol}
\usepackage{multirow}
\usepackage{graphicx}
\usepackage{wrapfig}

\usepackage{textcomp}
\usepackage{xcolor}
\usepackage{bm}
\usepackage{bbm}
\usepackage[varbb]{newpxmath}
\usepackage{algorithm}
\usepackage{algpseudocode}

\usepackage{amsmath,amssymb,amsfonts}

\usepackage{multirow}
\usepackage{caption}
\usepackage{subcaption}
\usepackage{supertabular}
\usepackage{colortbl}

\newcommand{\proposed}{{PIA}}
\newtheorem{lemma}{Lemma}
\newtheorem{corollary}{Corollary}
\definecolor{Gray}{gray}{0.92}

\begin{document}

\title{Perspective-Invariant Attack with Enhanced Transferability of Adversarial Examples}

\author{
        Kaisheng Liang,
        Yiming Cao,
        Bin Xiao,~\IEEEmembership{Fellow,~IEEE} 
\IEEEcompsocitemizethanks{
\IEEEcompsocthanksitem  
This work was supported in part by HK RGC GRF Grant PolyU 15201323. 

Kaisheng Liang, Yiming Cao, and Bin Xiao (corresponding author) are with the Department of Computing, The Hong Kong Polytechnic University, Hong Kong. (Email: kaisheng.liang@connect.polyu.hk, yiming.cao@connect.polyu.hk, b.xiao@polyu.edu.hk)
\IEEEcompsocthanksitem
\copyright{} 2026 IEEE. Personal use of this material is permitted. Permission from IEEE must be obtained for all other uses, in any current or future media, including reprinting/republishing this material for advertising or promotional purposes, creating new collective works, for resale or redistribution to servers or lists, or reuse of any copyrighted component of this work in other works.
}
}

\maketitle
\begin{abstract}
Adversarial examples generated on a surrogate deep neural network (DNN) can often successfully fool other black-box DNN models.
This cross-model transferability poses serious security threats to DNNs in practical applications.
Input transformation techniques are widely used to enhance adversarial transferability by increasing the diversity of input images.
However, existing methods primarily rely on local operations with limited degrees of freedom (DOF), such as block-wise shuffling and resizing, overlooking global perspective transformations that naturally arise from viewpoint changes.
In this work, we propose a Perspective-Invariant Attack (\proposed{}), which introduces a multi-DOF vertex sampling strategy that systematically covers the perspective transformation hierarchy from 2-DOF translation to 8-DOF projective mapping.
By generating geometrically diverse input variations, \proposed{} effectively reduces overfitting of adversarial perturbations to the surrogate model, thereby improving adversarial transferability.
We further propose PIA-Mix, a generic extension that maintains a complementary transformation pool and efficiently combines our perspective transformation with auxiliary methods for improved transferability.
Extensive experiments involving various DNN architectures, advanced defense mechanisms, and multimodal large language models (LLMs) demonstrate that \proposed{} and PIA-Mix outperform state-of-the-art transfer-based attacks.
\end{abstract}


\section{Introduction}

Deep neural networks (DNNs)~\cite{krizhevsky2012imagenet,He_2016_CVPR,SzegedyVISW16} have achieved remarkable success in the past decade and have been widely deployed in various critical domains, including autonomous driving~\cite{geiger2012we}, biomedical imaging~\cite{ronneberger2015u}, and face recognition~\cite{schroff2015facenet}.
Despite their impressive performance, DNNs have been shown to be vulnerable to adversarial examples~\cite{GoodfellowSS14,KurakinGB17a,eykholt2018robust,MadryMSTV18}, which are carefully crafted inputs with imperceptible perturbations that mislead model predictions.
More concerning is their cross-model transferability~\cite{XieZZBWRY19,DongLPS0HL18,Wang_2021_admix}, where adversarial examples crafted on one surrogate model can effectively fool other unknown target models without access to their parameters or architectures.
This transferability property makes adversarial examples particularly threatening, as attackers can leverage surrogate models to launch practical black-box attacks against deployed DNN systems.

Adversarial attacks are generally divided into white-box and black-box settings, depending on the attacker's level of access to the target model's information~\cite{GoodfellowSS14,moosavi2016deepfool,carlini2017towards,BrendelRB18_2018_iclr,DongLPS0HL18,XieZZBWRY19,Wang_2021_admix}.
In white-box scenarios, the adversary has full knowledge of the model architecture and parameters.
In contrast, black-box attacks operate without such access, presenting a significantly harder problem.
A widely studied type of black-box attack is the transfer-based attack~\cite{DongLPS0HL18,XieZZBWRY19,Wang_2021_admix}, which exploits the transferability of adversarial examples. Adversarial inputs generated on a white-box surrogate model can often deceive unknown black-box models. However, a key limitation of this approach is that adversarial examples tend to overfit to the surrogate model.
As a result, the success rate of such attacks may degrade significantly when transferred to black-box models with different structures.

Input transformation-based techniques are widely used in transfer-based adversarial attacks for improving adversarial transferability.
Representative methods include DI~\cite{XieZZBWRY19}, SI~\cite{LinS00H20}, Admix~\cite{Wang_2021_admix}, and SSA~\cite{long2022frequency}.
Recently, SIA~\cite{wang2023structure}, BSR~\cite{wang2024boosting}, and SID~\cite{zhou2025leveraging} achieved state-of-the-art results through block-wise shuffling, rotated block-wise shuffling, and multi-scale resizing with local image fusion, respectively.
These attacks rely on the idea that applying effective transformations increases image diversity~\cite{wang2023structure}, which helps mitigate overfitting to the surrogate model and improves transferability.
As shown in Figure~\ref{fig:intro_pia}, existing methods such as SIA and BSR focus on local block-wise operations, while SID applies multi-scale resizing with random padding and block-level image fusion. However, these methods rely on spatial augmentations within the image plane and do not consider global perspective transformations that naturally arise from viewpoint changes.
We argue that incorporating the perspective transformation family as a global geometric augmentation can further improve input diversity and thus enhance adversarial transferability.

In this work, we propose a Perspective-Invariant Attack (\proposed{}) to enhance adversarial transferability by making adversarial examples invariant to perspective transformations.
Specifically, we introduce a multi-DOF vertex sampling strategy that systematically covers the perspective transformation hierarchy, from 2-DOF translation to 8-DOF projective mapping.
By randomly sampling transformation modes and corresponding vertex positions, our method generates geometrically diverse input variations that go beyond the in-plane spatial augmentations of existing methods.
Furthermore, we propose PIA-Mix, a generic extension that maintains a complementary transformation pool. PIA-Mix treats our perspective transformation as the primary augmentation while incorporating auxiliary input transformation methods through random sampling, achieving further transferability gains with minimal computational overhead.
We conduct extensive experiments in different scenarios, including target DNNs with different architectures, combinations with existing attack techniques, advanced defensive mechanisms, and multimodal large language models (LLMs). As shown in Figure~\ref{fig:intro_performance}, \proposed{} and PIA-Mix achieve superior performance compared to existing methods across all evaluated settings.

\begin{figure}[t]
    \centering
    \includegraphics[width=0.85\linewidth]{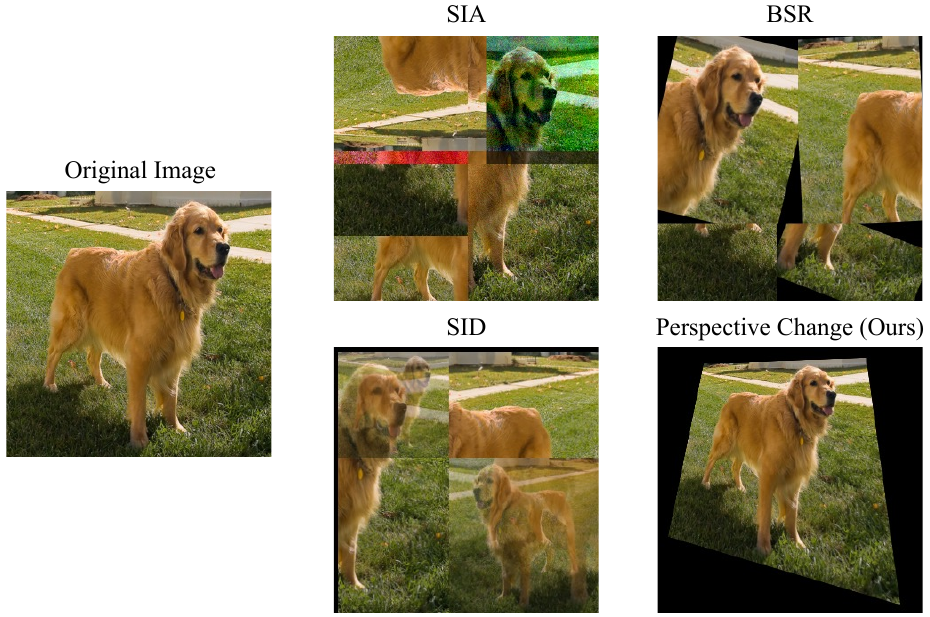}
    \caption{
        Comparison of input transformations in state-of-the-art attacks, including SIA~\cite{wang2023structure}, BSR~\cite{wang2024boosting}, and SID~\cite{zhou2025leveraging}.
        SIA performs block-wise shuffling, BSR combines shuffling with slight rotations, and SID applies multi-scale resizing with random padding and local image fusion.
        These methods rely on spatial augmentations within the image plane and do not consider global perspective transformations. In contrast, our method introduces diverse perspective transformations that cover the full projective transformation family.
    }
    \label{fig:intro_pia}
\end{figure}

\begin{figure}[t]
    \centering
    \includegraphics[width=0.85\linewidth]{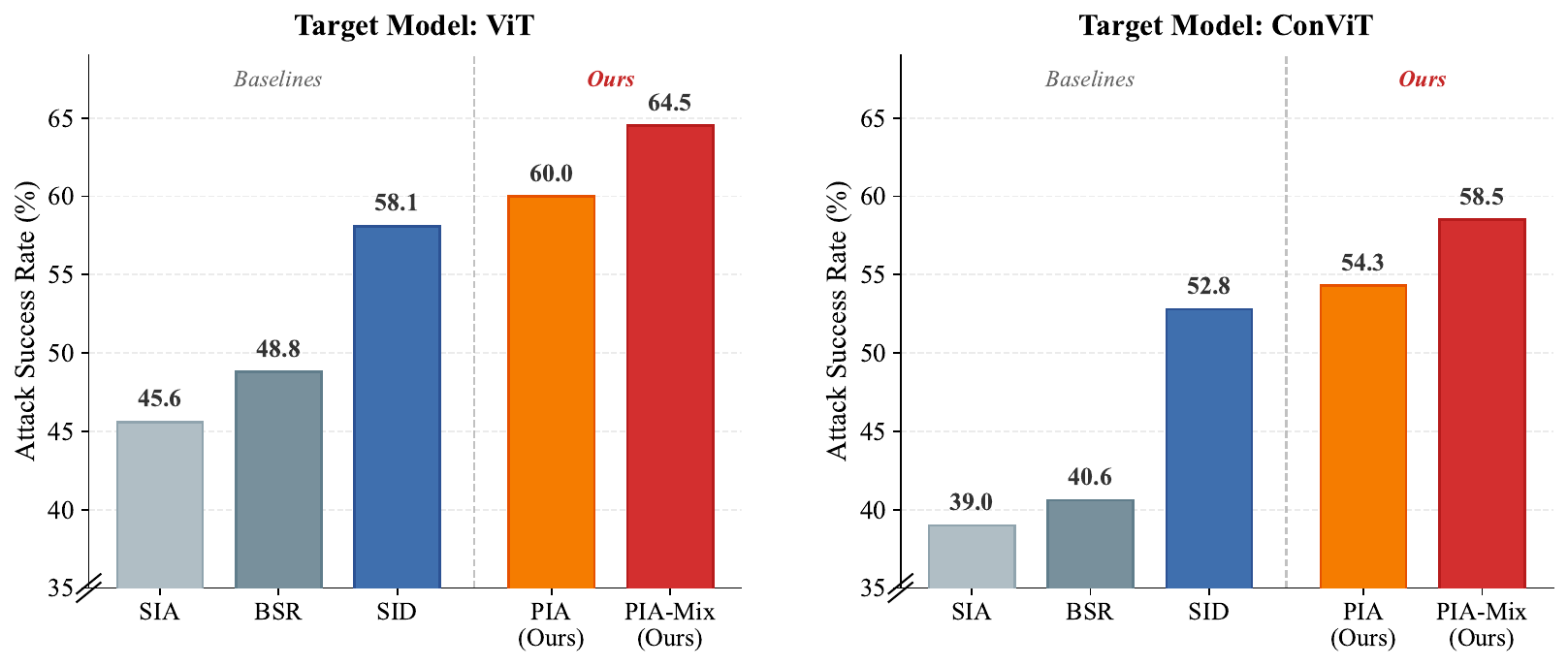}
    \caption{Average attack success rates (\%) of different methods using ResNet-50 as the surrogate model.
    Our PIA and PIA-Mix outperform existing methods.}
    \label{fig:intro_performance}
\end{figure}

In summary, this work makes the following contributions:
\begin{itemize}
    \item We revisit existing input transformation-based attacks and identify that recent methods rely on spatial augmentations within the image plane, overlooking global perspective transformations that naturally arise from viewpoint changes.
    \item We propose Perspective-Invariant Attack (\proposed{}), which introduces a multi-DOF vertex sampling strategy to systematically explore the perspective transformation hierarchy. We further propose PIA-Mix, which combines \proposed{} with complementary transformations via a configurable transformation pool for additional performance gains.
    \item Comprehensive experiments demonstrate that \proposed{} and PIA-Mix significantly outperform prior methods.
    Empirical results also show that our methods can be easily and effectively integrated with other attack techniques and frameworks.
\end{itemize}

\section{Related Work}

Although DNNs have made significant progress in tasks such as image classification, segmentation, and detection, adversarial attacks~\cite{GoodfellowSS14,SzegedyZSBEGF13} have revealed their vulnerability: adding small perturbations to the input can cause DNNs to make incorrect predictions.
To investigate the vulnerability of DNNs, researchers have proposed various adversarial attack methods aimed at generating adversarial examples that can mislead DNNs in different scenarios, including targeted attacks~\cite{zhao2021success,liang2024improving}, generative attacks~\cite{poursaeed2018generative,naseer2019cross}, physical-world attacks~\cite{eykholt2018robust,li2025uvattack}, query-based attacks~\cite{BrendelRB18_2018_iclr,ilyas2018black}, and transfer-based attacks~\cite{DongLPS0HL18,XieZZBWRY19}.
Among these, transfer-based attacks refer to scenarios where attackers use adversarial examples generated from a white-box surrogate model to successfully attack black-box models, posing serious security threats to DNNs.
This paper focuses on transfer-based attacks~\cite{Wang_2021_admix,Wu0X0M20,GuoLC20,liang2023styless,HuangKGHBL19,dai2024advdiff}, particularly on improving the transferability of adversarial examples when there are significant differences in model architectures.

\subsection{Transfer-based Adversarial Attacks}

Transfer-based adversarial attacks commonly build upon the Iterative Fast Gradient Sign Method (I-FGSM)~\cite{KurakinGB17a}, an extension of the Fast Gradient Sign Method (FGSM)~\cite{GoodfellowSS14}.
While I-FGSM shows strong performance in white-box settings with knowledge of target models, it performs poorly in black-box scenarios where target models are unknown.
To address this limitation, the Momentum-based I-FGSM (MI-FGSM)~\cite{DongLPS0HL18} incorporates momentum terms to stabilize gradient updates.
Nevertheless, a considerable performance gap persists between white-box and black-box attack scenarios, highlighting the need for more powerful transfer-based methods.

Input transformation-based attacks are a class of methods that aim to generate diverse input images to enhance adversarial transferability.
Diverse Input (DI)~\cite{XieZZBWRY19} employs random resizing and padding operations to generate diverse input images. Translation-Invariant (TI)~\cite{DongPSZ19} focuses on the translation property of input images and introduces a convolutional kernel to process gradients.
Scale-Invariant (SI)~\cite{LinS00H20} leverages the loss preservation property of pixel value scaling and utilizes multiple scaled images for data augmentation. The Admix method~\cite{Wang_2021_admix} combines small portions of other images to create multiple transformed images.
Spectrum Simulation Attack (SSA)~\cite{long2022frequency} perturbs the frequency domain of input images using Gaussian noise and random masking. 
DeCoWA~\cite{lin2024boosting} applies deformation-constrained warping to generate semantically consistent yet locally diverse inputs.
SIA~\cite{wang2023structure} introduces block-wise shuffling combined with random input transformations on each block, while BSR~\cite{wang2024boosting} extends the block shuffling concept by incorporating slight rotations of individual blocks.
Recently, SID~\cite{zhou2025leveraging} exploits the spatial invariance of DNNs through multi-scale and multi-position input transformations.
Beyond fixed transformation designs, some methods adopt search-based frameworks to explore optimal transformation strategies within a large transformation space.
L2T~\cite{zhu2024learning} proposes a learning-based framework to search for the optimal transformation for each input image.
OPS~\cite{guo2025boosting} performs stochastic optimization over surrogate model neighborhoods induced by input transformations and perturbations.
However, these methods typically incur significantly higher computational costs than single-transformation approaches and still rely on a predefined set of input transformations, making their performance dependent on the choice of effective base transformations.

In addition to input transformation-based attacks, researchers have proposed various other transfer-based attacks. 
Ensemble-based methods~\cite{DongLPS0HL18,LiuCLS17} leverage multiple surrogate models simultaneously to generate more transferable adversarial examples.
Ghost networks~\cite{li2020learning} simulate various surrogate networks using perturbed skip connections and dropout layers.
Transferable Adversarial Perturbation (TAP)~\cite{ZhouHCTHGY18} enhances transferability by maximizing the feature-space distance between benign and adversarial examples.
Skip Gradient Method (SGM)~\cite{Wu0X0M20} analyzes ResNet-based architectures and incorporates more gradients from skip connections.
LinBP~\cite{GuoLC20} introduces skipping of activation layers during the backpropagation process.
Feature Importance Aware (FIA) attack~\cite{Wang_2021_ICCV} focuses on corrupting critical features at the feature level, while Neuron Attribution-based Attack (NAA)~\cite{zhang2022improving} proposes a feature-level attack by utilizing advanced neuron attribution.

\subsection{Defenses against Adversarial Examples}

The High-Level Representation Guided Denoiser (HGD)~\cite{liao2018defense} proposes a denoising network trained to restore clean images from perturbed ones.
Adversarial Training (AT)~\cite{MadryMSTV18} enhances the adversarial robustness of DNNs by incorporating adversarial examples into the training process.
Building upon AT, Ensemble Adversarial Training~\cite{tramer2017ensemble} leverages multiple surrogate models to achieve better attack performance and robustness.
To address the computational cost of adversarial training, Fast Adversarial Training~\cite{wong2020fast} demonstrates that using single-step attacks can provide reasonable robustness with significantly reduced training time.
Random Smoothing (RS)~\cite{cohen2019certified} represents a major advance in certified defenses by providing provable robustness guarantees against adversarial perturbations.
The Neural Representation Purifier (NRP)~\cite{naseer2020self} employs self-supervised learning techniques to train a denoising network that can purify adversarial examples.
More recently, Diffusion-based Purification (DiffPure)~\cite{nie2022DiffPure} leverages pre-trained diffusion models to effectively remove adversarial perturbations through the diffusion process.

\section{Methodology}

\subsection{Threat Model and Problem Formulation}

\textbf{Threat Model.} We adopt the standard digital, transfer-based, black-box setting used in prior input-transformation attacks~\cite{XieZZBWRY19,wang2023structure}: the attacker has white-box access only to a surrogate model $F$ (a pre-trained image classifier) but has no knowledge of the target models, including their architectures, parameters, and training data. Adversarial examples $x^{\text{adv}}$ are crafted on $F$ under an $\ell_{\infty}$ perturbation budget and applied in the digital image domain.
The adversary's goal is to cause the unknown target models to misclassify $x^{\text{adv}}$.
The viewpoint analogy serves only as geometric motivation, and our perspective transformation is applied purely as a digital input transformation.

\textbf{Problem Formulation.} Given a benign image $x$, a surrogate model $F$, and a ground-truth label $y$, the optimization objective on the surrogate model in an untargeted setting can be formulated as:
\begin{equation}
    \arg\min_{x^{\text{adv}}} \mathcal{L}(F(x^{\text{adv}}), y), \quad s.t. \quad \|x - x^{\text{adv}}\|_{\infty} \leq \epsilon,
    \label{eq:threat_model}
\end{equation}
where $\mathcal{L}$ is the adversarial loss (e.g., negative cross-entropy loss) and $\epsilon$ is the perturbation budget.

To solve the problem above, existing methods usually adopt the iterative gradient-based technique, e.g., the MI-FGSM~\cite{DongLPS0HL18}, formulated as follows:
\begin{align}
    &g_{i+1} = \nabla_x \mathcal{L}(F(x^{\text{adv}}_i), y), \label{eq:i_fgsm_grad} \\ % [10pt]
    &\tilde g_{i+1} = \mu \cdot \tilde g_{i} + g_{i+1}/\| g_{i+1}\|_1, \label{eq:i_fgsm_momentum} \\
    &x^{\text{adv}}_{i+1} = x^{\text{adv}}_i - \eta \cdot \text{sign} \left( \tilde g_{i+1} \right), \label{eq:i_fgsm_update} \\ % [10pt]
    &x^{\text{adv}}_{i+1} = \text{Clip}_{x, \epsilon} \left( x^{\text{adv}}_{i+1} \right), \label{eq:i_fgsm_clip}
\end{align}
where $g_i$ is the gradient of the loss function with respect to the input image $x$ in the $i$-th iteration.
$\mu$ is the momentum decay factor, and $\eta$ is the step size.
$x^{\text{adv}}_{i+1}$ is clipped within the $\epsilon$ bound.
Initially, we set $x^\text{adv}_0=x$ and $\tilde g_0=0$.

\subsection{Motivation}
\label{sec:motivation}

Input transformation-based attacks improve adversarial transferability by making input images more diverse. However, existing methods mainly use fixed geometric operations with limited degrees of freedom (DOF). For example, methods like resizing, translation, or block-wise shuffling explore only a very small transformation space.
They fail to capture the global perspective transformations induced by viewpoint changes.
As a result, the generated adversarial examples often overfit to this narrow space, limiting their ability to transfer to black-box models.
Changes in camera position or object angles create a natural hierarchy of geometric transformations. This hierarchy ranges from simple 2-DOF translation up to complex 8-DOF projective transformations. It is important to note that lower-DOF operations like translation and affine transformations are mathematically special cases of the broader perspective transformation family. Unlike manually designed image distortions, these perspective transformations greatly increase image diversity while preserving straight lines and the object's geometric structure.
This observation motivates us to fully utilize the high-DOF perspective transformation family for generating adversarial examples. By exploring the entire hierarchy of perspective changes, ranging from simple low-DOF cases to the full 8-DOF projective mapping, we can overcome the geometric limitations of previous attacks. 
As illustrated in Figure~\ref{fig:motive_dof}, as the degrees of freedom increase (e.g., from 2 to 8 DOF), the transformed images become progressively more diverse, leading to stronger adversarial transferability.

\begin{figure}[htbp]
    \centering
    \begin{subfigure}{0.49\linewidth}
        \centering
        \includegraphics[width=\linewidth]{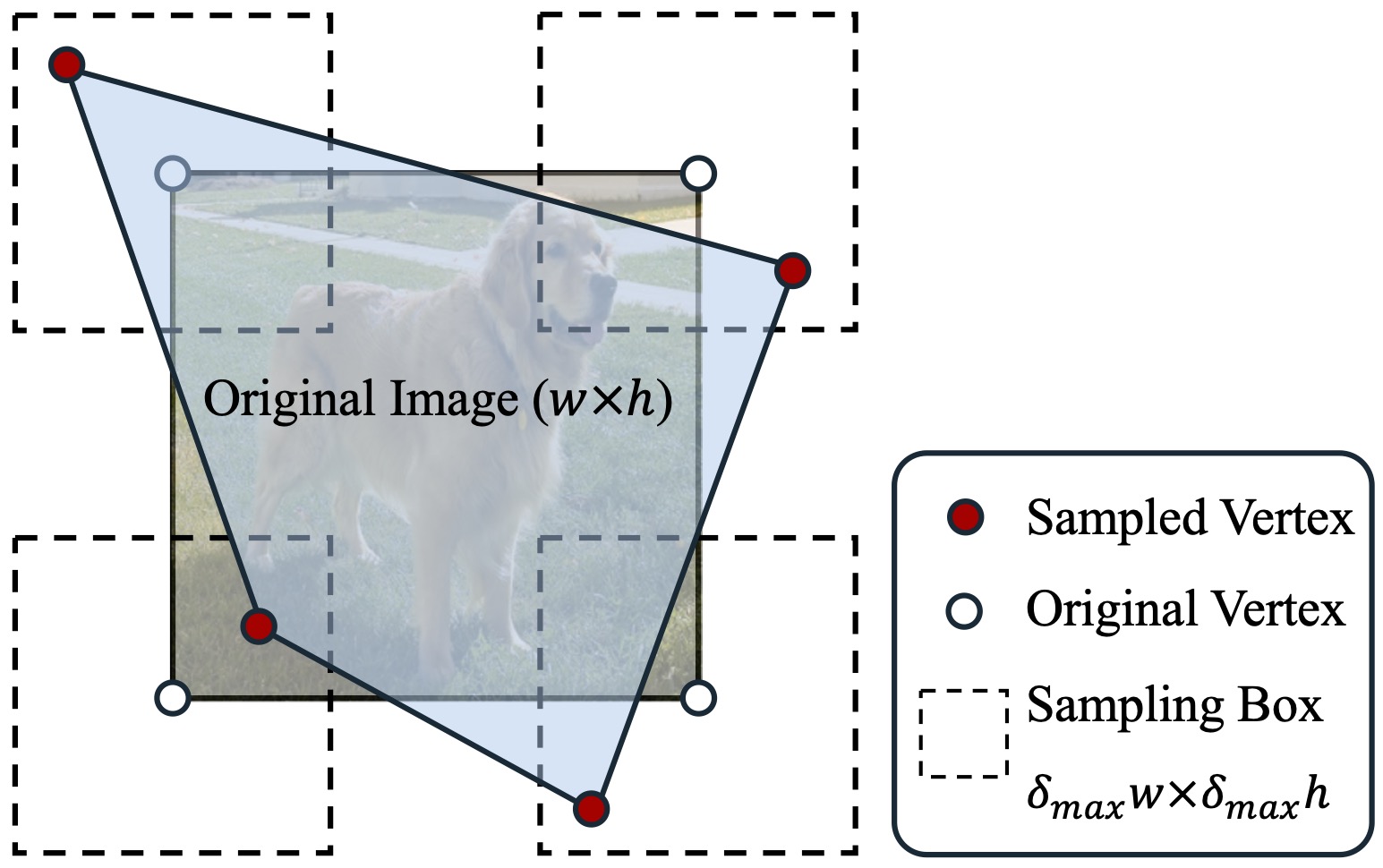}
        \caption{}
        \label{fig:motive_sampling}
    \end{subfigure}
    \hfill
    \begin{subfigure}{0.49\linewidth}
        \centering
        \includegraphics[width=\linewidth]{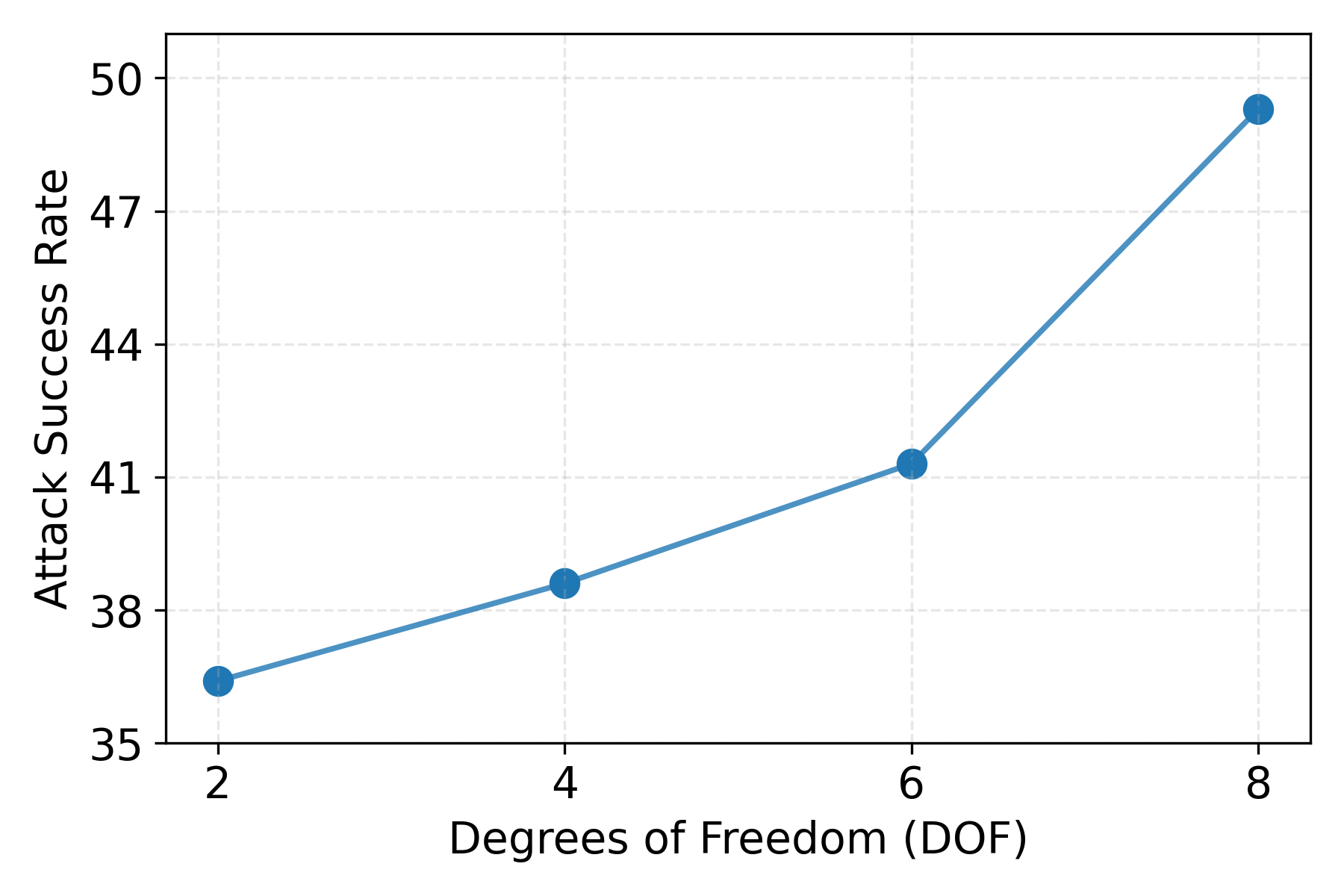}
        \caption{}
        \label{fig:motive_dof_b}
    \end{subfigure}
    \caption{
    \textbf{(a)}~Vertex sampling, where $\delta_{max}$ is the distortion budget. Each corner is sampled within the sampling box around its original vertex and may fall outside the image; out-of-frame regions are then cropped to the original size.
    \textbf{(b)}~Attack success rates (\%) under different DOF (ResNet-50 and ConViT as surrogate and target). As the geometric DOF increases from 2 to 8, adversarial transferability consistently improves.
    }
    \label{fig:motive_dof}
\end{figure}

A natural question arises: should we use transformations with even higher degrees of freedom, such as non-rigid operations like thin plate splines (TPS) or elastic deformations? We conjecture that while global parametric transformations provide stable gradient directions by preserving geometric structures, the unconstrained warping introduced by non-rigid operations might disrupt these directions, leading to suboptimal adversarial transferability. To validate this, we conduct a systematic comparison in Section~\ref{sec:compare_non_rigid}.
As shown in Table~\ref{tab:motive_non_rigid} below, the experimental results suggest that increasing geometric flexibility beyond projective transformations does not necessarily improve attack performance, highlighting the effectiveness of the 8-DOF perspective transformation as a well-balanced choice.

\begin{table}[htbp]
    \centering
    \resizebox{0.8\linewidth}{!}{%
    \begin{tabular}{lcc}
        \toprule[0.15em]
        Method & Transformation Type &  ASR (\%) \\
        \midrule
        Baseline (MI-FGSM) & None            & \textit{13.7} \\
        \midrule
        Affine          & Global (6-DOF)      &  41.3 \\
        Elastic         & Non-rigid (High)   &   23.0 \\
        TPS             & Non-rigid (High)   &   37.5 \\
        \midrule
        \rowcolor{gray!10} % 
        \textbf{Perspective (Ours)} & \textbf{Global (8-DOF)} & \textbf{49.3 (Highest)} \\
        \bottomrule[0.15em]
    \end{tabular}%
    }
    \caption{
    Attack success rates (\%) under different transformation types, using ResNet-50 and ConViT as the surrogate and target models, respectively. While non-rigid transformations (e.g., Elastic and TPS) introduce higher flexibility, they do not yield better transferability than global parametric transformations. Our 8-DOF perspective transformation achieves the highest success rate.
    }
    \label{tab:motive_non_rigid}
\end{table}

\subsection{Proposed Method}
In this section, we first introduce perspective transformation and then present our perspective-invariant attack (PIA) and PIA-Mix.
Figure~\ref{fig:method_illustration} provides an overview of our approach: the upper part illustrates the overall attack pipeline, where transformations are sampled from a Complementary Transformation Pool to generate diverse inputs for gradient computation; the lower part details our core perspective transformation, including the Multi-DOF vertex sampling strategy and the procedure for solving the transformation matrix.
In contrast to existing transformation-based attacks that typically apply a single geometric operation with limited DOF (e.g., resizing or translation), our method introduces two key designs: (1)~\textit{Multi-DOF Vertex Sampling}, which systematically explores the full hierarchy of perspective transformations from 2 to 8~DOF, covering a significantly richer geometric space; and (2)~the \textit{Complementary Transformation Pool}, which unifies our perspective transformation with auxiliary methods to exploit their complementary strengths without additional computational overhead.
In the following, we present each component from the ground up, starting with the perspective transformation itself and building toward the complete PIA and PIA-Mix frameworks.

\begin{figure*}[!t]
    \centering
    \includegraphics[width=0.9\linewidth]{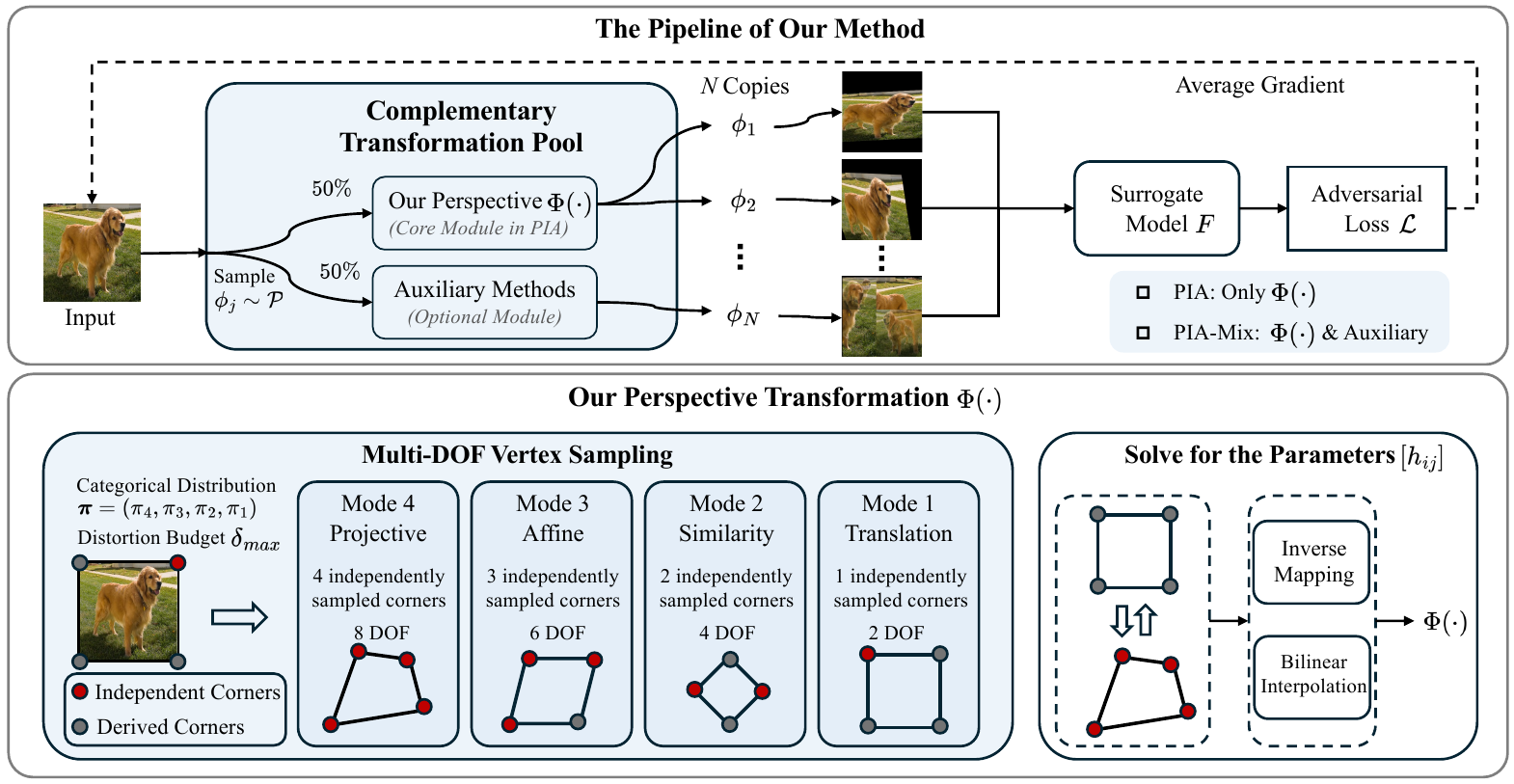}
    \caption{
    Overview of the proposed method.
    \textbf{Upper:} The pipeline of PIA-Mix. Transformations are sampled from a Complementary Transformation Pool~$\mathcal{P}$ containing our~$\Phi(\cdot)$ and auxiliary methods. Gradients over $N$ transformed copies are averaged for momentum-based updates. When $\mathcal{P}$ contains only $\Phi(\cdot)$, it reduces to PIA.
    \textbf{Lower:} Our Perspective Transformation~$\Phi(\cdot)$. The left panel shows Multi-DOF Vertex Sampling: from Mode~4 to Mode~1, the number of independently sampled corners decreases from 4 to 1, yielding 8 to 2 DOF. The right panel derives the $3\times3$ transformation matrix $[h_{ij}]$ from the sampled corners.
    }
    \label{fig:method_illustration}
\end{figure*}

\subsubsection{Perspective Transformation}
This transformation, also known as the projective transformation, can be represented as a $3 \times 3$ matrix $[h_{ij}]_{3 \times 3}$~\cite{szeliski2022computer}:

\begin{equation}
\begin{bmatrix} \hat{u} \\ \hat{v} \\ \hat{w} \end{bmatrix} =
\begin{bmatrix} h_{11} & h_{12} & h_{13} \\ h_{21} & h_{22} & h_{23} \\ h_{31} & h_{32} & 1 \end{bmatrix}
\begin{bmatrix} u \\ v \\ 1 \end{bmatrix},
\label{eq:perspective_transformation}
\end{equation}
\begin{equation}
\begin{aligned}
u' = \frac{\hat{u}}{\hat{w}} &= \frac{h_{11} u + h_{12} v + h_{13}}{h_{31} u + h_{32} v + 1}, \\
v' = \frac{\hat{v}}{\hat{w}} &= \frac{h_{21} u + h_{22} v + h_{23}}{h_{31} u + h_{32} v + 1},
\label{eq:perspective_transformation_homogenization}
\end{aligned}
\end{equation}
where $(u, v)$ are the coordinates of the original image, $(u', v')$ are the coordinates of the transformed image, and $h_{ij}$ are the elements of the transformation matrix.
To address the issues of holes and pixel overlapping in the transformed image, we can employ inverse mapping to establish a correspondence between the transformed and original images.
Specifically, for each pixel coordinate $(u', v')$ in the transformed image, we compute its corresponding position $(u, v)$ in the original image by applying the inverse transformation matrix.
Since the computed coordinates $(u, v)$ may not be integer coordinates, we can utilize interpolation techniques (e.g., bilinear interpolation) to estimate the pixel value at $(u, v)$.

The commonly used transformations such as scaling, rotation, and translation can be represented as special cases of the projective transformation.
For example, $\begin{bmatrix} h_{11} & h_{12} \\ h_{21} & h_{22} \end{bmatrix}$ defines scaling and rotation, while parameters $h_{13}$ and $h_{23}$ represent translation. Existing transformation-based attacks have proved that scaling~\cite{XieZZBWRY19}, translation~\cite{DongPSZ19}, and rotation~\cite{wang2024boosting} are effective for enhancing the transferability of adversarial examples.
However, the perspective transformation still remains under-explored in the context of adversarial attacks.
Specifically, the parameters $h_{31}$ and $h_{32}$ control the perspective of the transformed image, which is not considered in existing input transformation-based attacks.

\subsubsection{Randomized Four-Corner Parameterization}
Based on Equation~\ref{eq:perspective_transformation}, the degree of freedom (DOF) of the perspective transformation is eight.
Thus, given four pairs of corresponding points in a 2D image, we can solve for the transformation matrix.
Specifically, we randomly sample four positions in the input image and use them as the four corners of the transformed image.
We formulate the four randomized corner positions as follows:
\begin{equation}
    \begin{aligned}
    u_{i,j}^\prime &= u_{i,j} + \operatorname{int} \left( \xi_{i,j}^h \cdot (\delta_{max} \cdot \frac{h}{2}) \right), \\
    v_{i,j}^\prime &= v_{i,j} + \operatorname{int} \left( \xi_{i,j}^w \cdot (\delta_{max} \cdot \frac{w}{2}) \right), \\
    &\text{where } i \in \{0,h-1\}, j \in \{0,w-1\},
    \end{aligned}
    \label{eq:perspective_transformation_corner}
\end{equation}
where $\xi_{i,j}^h$ and $\xi_{i,j}^w$ are randomly sampled from the uniform distribution $U(-1, 1)$; 
$h$ and $w$ are the height and width of the input image, respectively;
$\delta_{max}$ is the maximum distortion budget;
$\delta_{max} \cdot \frac{h}{2}$ and $\delta_{max} \cdot \frac{w}{2}$ are the maximum displacements of the corner positions;
$(u_{i,j}, v_{i,j})$ represents a corner position of the input image;
$u_{i,j}^\prime$ and $v_{i,j}^\prime$ are the coordinates of the transformed corner position.
Note that the transformed corner coordinates $(u_{i,j}^\prime, v_{i,j}^\prime)$ are allowed to fall outside the original image boundaries. The resulting transformed image is then simply cropped to maintain the original spatial dimensions.

\subsubsection{Multi-DOF Vertex Sampling}
In the above formulation, Equation~\ref{eq:perspective_transformation_corner} samples four corner positions independently. However, this pure 4-point sampling is strongly biased toward the generic projective space, leaving lower-degree-of-freedom (DOF) transformations (e.g., affine or pure translation) heavily under-represented.
In fact, four independently sampled corners form a parallelogram (i.e., yield an affine map with $h_{31}{=}h_{32}{=}0$) only with vanishingly small probability:

\begin{lemma}
\label{lem:degenerate}
Let four points be sampled independently, with the $k$-th point drawn uniformly from an $N_x \times N_y$ grid of integer positions centered at the $k$-th vertex of a parallelogram. Then the probability that the four points form a parallelogram is
\begin{equation}
    P_{\mathrm{aff}}(N_x, N_y) = \frac{2N_x^2+1}{3N_x^3} \cdot \frac{2N_y^2+1}{3N_y^3}.
    \label{eq:p_aff}
\end{equation}
\end{lemma}

\noindent\textit{Proof.}
Label the corners cyclically, so $\{P_1,P_3\}$ and $\{P_2,P_4\}$ are the two pairs of opposite vertices; a quadrilateral is a parallelogram iff its diagonals share a midpoint, i.e., $P_1{+}P_3=P_2{+}P_4$. As the box centers already form a parallelogram, this reduces to $a_1{+}a_3=a_2{+}a_4$ for the $x$-offsets $a_k$ (and independently for the $y$-offsets). On one axis, $a_1{+}a_3$ and $a_2{+}a_4$ are i.i.d.\ sums of two uniform offsets and hence follow the triangular pmf $p_k=(N-|k|)/N^2$, so they coincide with probability $\sum_k p_k^2 = \frac{1}{N^4}\sum_{|k|<N}(N-|k|)^2 = \frac{2N^2+1}{3N^3}$. Multiplying the independent $x$- and $y$-axes yields Equation~\ref{eq:p_aff}.
\hfill$\square$

\begin{corollary}
\label{cor:degenerate}
Under Equation~\ref{eq:perspective_transformation_corner} with budget $\delta_{max}$, each corner is sampled from a box of $N_x \approx \delta_{max} w$ by $N_y \approx \delta_{max} h$ integer positions, so $P_{\mathrm{aff}} \approx 4/(9 w h \delta_{max}^2)$. For our setting ($\delta_{max}{=}0.6$, $w{=}h{=}299$), $P_{\mathrm{aff}} \approx 1.4 \times 10^{-5}$.
\end{corollary}

Hence pure 4-point sampling almost surely yields a genuine 8-DOF transformation and rarely covers lower-DOF cases.
To improve the coverage across the perspective transformation hierarchy, we introduce a multi-DOF vertex sampling strategy.
Specifically, we define four transformation modes, where Mode~$k$ ($k \in \{4,3,2,1\}$) randomly samples $k$ independent vertices and deterministically infers the remaining $4{-}k$ corners to enforce the corresponding geometric constraint.
We first draw a mode from the categorical distribution $\boldsymbol{\pi} = (\pi_4, \pi_3, \pi_2, \pi_1)$, where $\pi_k$ denotes the probability of selecting Mode~$k$.
Concretely, Mode~4 (8-DOF) samples all four corners independently, yielding a general 8-DOF projective transformation.
Mode~3 (6-DOF) samples three corners independently and computes the fourth as the vertex that completes a parallelogram.
Mode~2 (4-DOF) samples two diagonally opposite corners and constructs the other two to form a rectangle.
Mode~1 (2-DOF) samples a single displacement vector and applies it uniformly to all four corners, reducing to a 2-DOF translation.

Regardless of the sampled mode, the process ultimately yields four corner endpoints, which allows us to formulate the unified perspective transformation function $\Phi(\cdot)$ as:
\begin{equation}
    \Phi(x; [h_{ij}]) = \Phi(x; \delta_{max}, {\pi}),
    \label{eq:perspective_transformation_function}
\end{equation}
where $[h_{ij}]$ is the derived perspective transformation matrix. Given an image $x$, once the four corners are determined by the distortion budget $\delta_{max}$ and the mode distribution ${\pi}$, we can obtain eight equations based on Equation~\ref{eq:perspective_transformation_homogenization}. By solving these equations, we compute the final transformation matrix $[h_{ij}]$.

\subsubsection{Perspective-Invariant Attack (PIA)}
Building upon the multi-DOF perspective transformation function $\Phi(\cdot)$, we propose our Perspective-Invariant Attack (PIA). The core idea is straightforward: by forcing the adversarial examples to remain misclassified across a wide range of geometric distortions, we prevent the perturbations from overfitting to the source model's specific feature alignments, thereby significantly improving cross-model transferability.

Specifically, we integrate $\Phi(\cdot)$ seamlessly into the standard momentum iterative framework (MI-FGSM). In each optimization step, we generate $N$ transformed copies of the current adversarial image $x^{\text{adv}}_{i}$ using our randomized $\Phi(\cdot)$. The gradients are then calculated and averaged over these geometric variants. Formally, Equation~\ref{eq:i_fgsm_grad} is reformulated as follows:
\begin{equation}
    g_{i+1} = \frac{1}{N} \sum_{j=1}^{N} \nabla_x \mathcal{L}(F(\Phi(x^{\text{adv}}_{i}; \delta_{max}, \pi)_j), y),
    \label{eq:pia_grad_pure}
\end{equation}
where $\Phi(\cdot)_j$ denotes the $j$-th randomly sampled perspective transformation instance. By solely replacing conventional augmentation techniques with our multi-DOF $\Phi(\cdot)$, PIA establishes a new state-of-the-art in transferability.

\subsubsection{PIA with Complementary Transformation Pool (PIA-Mix)}

While our proposed Perspective-Invariant Attack (PIA) generates highly transferable adversarial examples independently, its core mechanism explores the global geometric space.
This mechanism is inherently orthogonal to existing input transformations that focus on multi-scale spatial augmentations with local fusion (e.g., SID) or block-level spatial disruptions (e.g., SIA, BSR).

To seamlessly exploit the combinations of different transformations without introducing significant computational overhead, we propose a generic extension termed \textbf{PIA-Mix}, which maintains a \textit{Complementary Transformation Pool} $\mathcal{P}$. Rather than simply stacking augmentations, we treat $\mathcal{P}$ as a configurable distribution of distinct transformations. During the $N$-sample gradient averaging in each attack iteration, we sample transformations such that our perspective transformation $\Phi(\cdot)$ acts as the primary driver (e.g., selected with a 50\% probability), while the remaining slots are filled by randomly sampling from auxiliary methods in the pool. By keeping the total number of transformations $N$ strictly constant, PIA-Mix further improves adversarial transferability gains with minimal additional overhead.

Formally, we incorporate this pool-based strategy into the MI-FGSM framework, updating Equation~\ref{eq:pia_grad_pure} as follows:
\begin{equation}
    g_{i+1} = \frac{1}{N} \sum_{j=1}^{N} \nabla_x \mathcal{L}(F(\phi_j(x^{\text{adv}}_{i})), y), \quad \text{where } \phi_j \sim \mathcal{P},
    \label{eq:pia_grad}
\end{equation}
where $\mathcal{P}$ defines the sampling space comprising our $\Phi(\cdot)$ and other candidate methods. We sample $\phi_j$ from $\mathcal{P}$ according to a predefined probability distribution.

This generic formulation not only verifies that PIA fundamentally expands the attack surface beyond existing techniques, but also provides a flexible framework for future combinations. We provide detailed isolated studies and optimal combination analyses in our ablation experiments.

\subsubsection{PIA Algorithm}
We summarize the PIA-Mix algorithm in Algorithm~\ref{alg:pia}, and PIA can be seen as a special case of PIA-Mix when $\mathcal{P}$ only contains our perspective transformation $\Phi(\cdot)$.
In each iteration of PIA-Mix, we randomly transform the input using a set of transformation functions $\mathcal{P}$, including our perspective transformation $\Phi(\cdot)$, which has been overlooked in existing attack methods.
We sum the gradients of the transformed images to obtain the final gradient $g^\prime$.
Since $g^\prime$ will be normalized when calculating the momentum $g_{i+1}$, we can omit the averaging operation.
Finally, we update the intermediate adversarial example $x^{\text{adv}}_i$ using the sign of the momentum, and clip it to be $\ell_\infty$-bounded by $\epsilon$.

\begin{algorithm}[t]
    \algnewcommand\algorithmicinput{\textbf{Input:}}
    \algnewcommand\Input{\item[\algorithmicinput]}
    \algnewcommand\algorithmicoutput{\textbf{Output:}}
    \algnewcommand\Output{\item[\algorithmicoutput]}
    \caption{PIA-Mix Algorithm}
    \label{alg:pia}
	\begin{algorithmic}[1]
		\Input Surrogate model $F$, 
               clean image $x$, 
               true label $y$, iteration number $T$, perturbation budget $\epsilon$,
               step size $\eta$,
               decay factor $\mu$, height $h$, width $w$,
               number of transformed copies $N$ % per iteration
        \Input Maximum distortion budget $\delta_{max}$ and vertex mode distribution $\pi$ for our $\Phi(\cdot)$
        \Input Transformation pool $\mathcal{P}$ and their sampling probabilities
        \Output An adversarial example $x^{\text{adv}}$
        \State $x^{\text{adv}}_0=x, \tilde g_0=0$. 
		\For{$i = 0 \textbf{ to } T-1$}
            \State $g_{i+1}=0$ \Comment{Gradients of $N$ transformed images}
            \For{$j = 0 \textbf{ to } N-1$}
                \State Randomly select a function $\phi_j(\cdot)$ from $\mathcal{P}$
                \If{$\phi_j(\cdot)$ is our $\Phi(\cdot)$} \Comment{Get matrix $[h_{ij}]$}
                    \State Sample corner positions with $\delta_{max}$ and $\pi$
                    \State For each corner pair $(u,v)$ and $(u',v')$, obtain: 
                    \begin{equation}
                        \begin{aligned}
                        u' &= \frac{h_{11} u + h_{12} v + h_{13}}{h_{31} u + h_{32} v + 1}, \\
                        v' &= \frac{h_{21} u + h_{22} v + h_{23}}{h_{31} u + h_{32} v + 1},
                        \nonumber
                        \end{aligned}
                    \end{equation}
                    \State Collect eight equations and solve for $[h_{ij}]$
                \EndIf
            
                \State $g_{i+1} = g_{i+1} + \nabla_x \mathcal{L}(F(\phi_j(x^{\text{adv}}_i)), y)$
            \EndFor
            \State Calculate $\tilde g_{i+1} = \mu \cdot \tilde g_{i} + g_{i+1}/\| g_{i+1}\|_1$ \Comment{Momentum}
            \State Update example
		        $x^{\text{adv}}_{i+1} = x^{\text{adv}}_i - \eta \cdot \text{sign}(\tilde g_{i+1})$
            \State Clip $x^{\text{adv}}_{i+1}$ to be $\ell_{\infty}$-bounded by $\epsilon$
		\EndFor
        \State $x^{\text{adv}}=x^{\text{adv}}_T$
        \State \Return $x^{\text{adv}}$
	\end{algorithmic}
\end{algorithm}

\section{Experiments}

In this section, we first introduce the experimental setup in Section~\ref{sec:exp_setup}. 
We then systematically evaluate our methods by progressively answering the following questions. 
These questions are designed to assess PIA from multiple perspectives, 
ranging from standard transfer settings to more challenging scenarios 
such as defended models and multimodal systems:
(1) \textbf{Effectiveness} (Section~\ref{sec:exp_single}): 
Does PIA improve transferability over existing input transformation methods?
(2) \textbf{Compatibility} (Section~\ref{sec:exp_combine}): 
Can PIA be combined with other attack frameworks?
(3) \textbf{Robustness} (Section~\ref{sec:exp_defense}): 
Does PIA remain effective under defenses?
(4) \textbf{Generalization} (Section~\ref{sec:exp_generalization}): 
Does PIA generalize across datasets and multimodal systems?

\subsection{Experimental Setup}
\label{sec:exp_setup}

\textbf{Datasets.}
We evaluate our method on two widely-used benchmark datasets. Following previous work~\cite{DongLPS0HL18,wang2023structure,wang2024boosting}, the primary dataset is the ImageNet-compatible dataset from the NIPS 2017 adversarial attack challenge. It consists of 1,000 images from the ILSVRC 2012 dataset~\cite{russakovsky2015imagenet} specifically curated for adversarial attack evaluation. Additionally, we conduct experiments on CIFAR-10~\cite{krizhevsky2009learning}, where we construct a test set of 1,000 images by randomly sampling 100 images from each of the 10 classes.
The input image resolution is set to 299×299 for ImageNet and 32×32 for CIFAR-10.

\textbf{General Settings.}
We evaluate the transferability of adversarial examples in untargeted attack scenarios. We employ an $\ell_\infty$-norm bounded attack with perturbation budget $\epsilon=16/255$. The decay factor $\mu$ is set to 1.0, and step size $\eta$ is set to $2/255$. We run the iterative attack for 10 iterations ($T=10$). 
All methods utilize the negative cross-entropy loss function for optimizing the adversarial examples.

\textbf{Models.}
For the ImageNet-compatible dataset, we evaluate our method on eight pre-trained DNN models, comprising four CNN-based models and four transformer-based architectures. The CNN-based models include ResNet-50 (RN-50) \cite{He_2016_CVPR}, DenseNet-121 (DN-121) \cite{huang2017densely}, MobileNet-v2 (MB-v2) \cite{sandler2018mobilenetv2}, and Inception-v3 (Inc-v3) \cite{SzegedyVISW16}.
For transformer-based architectures, we employ Vision Transformer (ViT) \cite{dosovitskiy2020vit}, LeViT \cite{graham2021levit}, ConViT \cite{d2021convit}, and Pooling-based Vision Transformer (PiT) \cite{heo2021rethinking}. 
We also use three secured DNNs: Inc-v3$_{ens3}$ (an ensemble of three Inception-v3 networks), Inc-v3$_{ens4}$ (an ensemble of four Inception-v3 networks), and IR-v2$_{ens}$ (an ensemble of three IncRes-v2 networks). These models were adversarially trained and have been widely used in previous work~\cite{DongLPS0HL18,DongPSZ19,Wang_2021_admix}.
For the CIFAR-10 dataset, we primarily evaluate our method on Inc-v3, MB-v2, and RN-20. Given that CIFAR-10's small image dimensions are not well-suited for vision transformers, we only use CNN-based architectures.
Additionally, we incorporate two adversarially trained models for CIFAR-10: RN-20-ens-gal~\cite{kariyappa2019improving} and RN-20-ens-dv \cite{yang2020dverge}.

\textbf{Attacks.}
We comprehensively evaluate our method against existing transfer-based attack techniques: DI \cite{XieZZBWRY19}, TI \cite{DongPSZ19}, DeCoWA \cite{lin2024boosting}, SIA \cite{wang2023structure}, BSR \cite{wang2024boosting}, and SID \cite{zhou2025leveraging}.
We also evaluate the attack performance in the search-based frameworks: L2T \cite{zhu2024learning} and OPS \cite{guo2025boosting}.
We incorporate MI~\cite{DongLPS0HL18} as a fundamental baseline component across all attack methods. For clarity in presentation, we omit the `MI' prefix when referring to these methods. All experiments strictly follow the original parameter settings proposed in their respective papers.

\textbf{Evaluation Metric.}
We use the Attack Success Rate (ASR) as our primary evaluation metric. ASR is defined as the percentage of generated adversarial examples that are successfully misclassified by the target model (i.e., the predicted label differs from the ground-truth label).

\textbf{Implementation Details of Our Methods.}
We use $\delta_{max}=0.6$ and vertex mode distribution $\pi=(0.4, 0.3, 0.2, 0.1)$ for our PIA.
For the transformation pool $\mathcal{P}$ in PIA-Mix, we include PIA and SID with probabilities (0.5, 0.5).
The number of transformed copies $N$ is set to 20.

% -------------------- Table 1 : single method --------------------
\begin{table*}[!t]
    \centering
    \resizebox{0.75\textwidth}{!}{%
    \begin{tabular}{clccccccccc}
    \toprule[0.15em]
    Source & Attack & 
    RN-50 & DN-121 & Inc-v3 & MB-v2 & ViT & LeViT & ConViT & PiT & Average
    \\ \midrule
    \multirow{8}{*}{RN-50} 
    & DI         &\textbf{100} & 97.0 & 76.1 & 95.3 & 29.7 & 32.4 & 19.8 & 29.8 & 60.0 \\
    & TI         &\textbf{100} & 90.5 & 59.8 & 87.2 & 31.6 & 29.0 & 18.9 & 25.7 & 55.3 \\
    & DeCoWA     &\textbf{100} & 99.7 & 95.9 & 99.4 & 53.3 & 68.6 & 41.1 & 59.2 & 77.2 \\
    & SIA        &\textbf{100} & \textbf{99.9} & 96.2 & 99.7 & 45.6 & 68.9 & 39.0 & 59.0 & 76.0 \\
    & BSR        &\textbf{100} & \textbf{99.9} & 95.5 & 99.6 & 48.8 & 71.6 & 40.6 & 60.7 & 77.1 \\
    & SID        &\textbf{100} & \textbf{99.9} & 97.9 & 99.6 & 58.1 & 82.1 & 52.8 & 67.4 & 82.2 \\
    \cmidrule(lr){2-11}
    & PIA (Ours)     & \textbf{100} & \textbf{99.9} & 98.1 & \textbf{99.8} & 60.0 & 81.9 & 54.3 & 69.3 & 82.9 \\
    & PIA-Mix (Ours) & \textbf{100} & \textbf{99.9} & \textbf{98.5} & \textbf{99.8} & \textbf{64.5} & \textbf{86.9} & \textbf{58.5} & \textbf{76.2} & \textbf{85.5} \\

    \midrule
    \multirow{8}{*}{DN-121} 
    & DI         &96.7 & \textbf{100} & 77.5 & 93.2 & 32.2 & 38.5 & 23.2 & 34.3 & 61.9 \\
    & TI         &90.5 & \textbf{100} & 62.9 & 86.4 & 33.9 & 35.1 & 22.2 & 33.3 & 58.0 \\
    & DeCoWA     &99.5 & \textbf{100} & 95.7 & 99.0 & 52.9 & 69.3 & 43.3 & 64.0 & 78.0 \\
    & SIA        &\textbf{99.7} & \textbf{100} & 96.6 & 99.2 & 50.5 & 73.1 & 42.9 & 65.5 & 78.4 \\
    & BSR        &99.5 & \textbf{100} & 95.0 & 98.9 & 49.6 & 71.8 & 41.7 & 63.7 & 77.5 \\
    & SID        &99.5 & \textbf{100} & 97.4 & 99.3 & 60.0 & 79.3 & 49.7 & 69.0 & 81.8 \\
    \cmidrule(lr){2-11}
    & PIA (Ours)  & 99.6 & \textbf{100} & 97.9 & 99.3 & 61.6 & 80.7 & 49.5 & 71.0 & 82.5 \\
    & PIA-Mix (Ours) & \textbf{99.7} & \textbf{100} & \textbf{98.3} & \textbf{99.4} & \textbf{64.2} & \textbf{83.8} & \textbf{58.0} & \textbf{77.2} & \textbf{85.1} \\

    \midrule
    \multirow{8}{*}{Inc-v3} 
    & DI         &74.0 & 74.7 & 99.6 & 79.5 & 29.5 & 30.8 & 20.8 & 29.3 & 54.8 \\
    & TI         &59.3 & 61.7 & \textbf{100} & 67.8 & 30.1 & 25.7 & 19.1 & 26.8 & 48.8 \\
    & DeCoWA    & 93.7 & 96.3 & \textbf{100} & 96.1 & 46.4 & 64.7 & 41.8 & 60.2 & 74.9 \\
    & SIA        &94.8 & 96.7 & \textbf{100} & 96.4 & 47.6 & 67.7 & 42.7 & 62.4 & 76.0 \\
    & BSR        &91.8 & 95.8 & \textbf{100} & 94.0 & 44.8 & 58.8 & 38.8 & 57.4 & 72.7 \\
    & SID  & 95.2 & 97.6 & \textbf{100} & 96.5 & 52.0 & 70.1 & 45.4 & 66.6 & 77.9 \\
    \cmidrule(lr){2-11}
    & PIA (Ours) &  95.3 & 97.3 & \textbf{100} & 96.9 & 54.5 & 71.4 & 47.7 & 68.4 & 78.9 \\
    & PIA-Mix (Ours) & \textbf{96.0} & \textbf{97.8} & \textbf{100} & \textbf{97.2} & \textbf{56.8} & \textbf{74.1} & \textbf{50.7} & \textbf{69.9} & \textbf{80.3} \\

    \midrule
    \multirow{8}{*}{MB-v2} 
    & DI         &89.8 & 88.3 & 70.0 & \textbf{100} & 28.2 & 34.0 & 17.9 & 30.1 & 57.3 \\
    & TI         &78.4 & 75.8 & 53.1 & \textbf{100} & 28.5 & 29.0 & 18.8 & 24.5 & 51.0 \\
    & DeCoWA & 99.0 & 99.2 & 94.8 & \textbf{100} & 48.6 & 73.4 & 40.8 & 65.4 & 77.7 \\
    & SIA        &99.2 & 98.8 & 90.2 & \textbf{100} & 41.1 & 65.7 & 32.6 & 53.0 & 72.6 \\
    & BSR        &99.1 & 99.3 & 95.2 & \textbf{100} & 47.2 & 73.6 & 38.7 & 62.5 & 77.0 \\
    & SID   & 99.3 & \textbf{99.5} & 95.4 & \textbf{100} & 55.2 & 78.0 & 46.2 & 71.4 & 80.6 \\
    \cmidrule(lr){2-11}
    & PIA (Ours) & 99.2 & \textbf{99.5} & 95.6 & \textbf{100} & 54.9 & 79.3 & 46.0 & 70.9 & 80.7 \\
    & PIA-Mix (Ours) & \textbf{99.5} & \textbf{99.5} & \textbf{96.8} & \textbf{100} & \textbf{57.6} & \textbf{84.4} & \textbf{50.8} & \textbf{73.3} & \textbf{82.7} \\

    \bottomrule[0.15em]
    \end{tabular}%
    }
    \caption{Attack success rates (\%) against eight target models on the ImageNet-compatible dataset.
    All methods are based on MI.
    The best results are shown in bold.}
    \label{tab:sole}
\end{table*}

\subsection{Comparison with Input Transformation Methods}
\label{sec:exp_single}

Table~\ref{tab:sole} shows the attack success rates of our methods against eight target models on the ImageNet-compatible dataset.
We use different surrogate models, including RN-50, DN-121, Inc-v3, and MB-v2.
We compare PIA and PIA-Mix with existing input-transformation-based attack methods: DI, TI, DeCoWA, SIA, BSR, and SID.
The framework of all methods is based on MI-FGSM, which is a widely used framework for adversarial attacks.
Our PIA and PIA-Mix consistently achieve the highest average attack success rates across all four surrogate model settings.
For example, when using RN-50 as the surrogate model, our PIA-Mix achieves an average attack success rate of $85.5\%$, which is $3.3\%$ higher than the next best method, SID ($82.2\%$).

Notably, our PIA involves only perspective transformation, an effective geometric transformation that has been overlooked in existing input-transformation-based attacks.
Despite using only a single type of transformation, PIA alone outperforms all existing methods across all surrogate settings.
Compared with BSR and SIA, two block-wise shuffling-based methods that also operate in the pixel space, PIA shows clear advantages. When using RN-50 as the surrogate model, PIA achieves an average ASR of $82.9\%$, compared to $76.0\%$ for SIA and $77.1\%$ for BSR, demonstrating the superiority of perspective transformation over block-wise shuffling.
Compared with SID, a concurrent method that leverages multi-scale resizing and local image fusion, PIA achieves slightly higher ASRs (e.g., $82.9\%$ vs.\ $82.2\%$ on RN-50, $82.5\%$ vs.\ $81.8\%$ on DN-121) using only a single type of geometric transformation, highlighting the strong potential of perspective transformation.

Furthermore, PIA-Mix, which combines PIA with SID via random selection, consistently outperforms both PIA and SID by notable margins (e.g., $85.5\%$ vs.\ $82.9\%$ and $82.2\%$ on RN-50).
This indicates that the diversity introduced by perspective transformation is complementary to SID's multi-scale spatial transformations, and the two can be effectively combined for stronger transferability.

% -------------------- Table 2 : DI-TI- --------------------
\begin{table*}[!t]
    \centering
    \resizebox{0.81\textwidth}{!}{%
    \begin{tabular}{c|ccccccccccc}
    \toprule[0.15em]
    Source &
    Category & Attack & 
    RN-50 & DN-121 & Inc-v3 & MB-v2 & ViT & LeViT & ConViT & PiT & Average
    \\ \midrule
    \multirow{12}{*}{RN-50} &
    \multirow{7}{*}{\shortstack[c]{Fixed\\Transformation\\Combination}} 
    &   MI & \textbf{100} & 84.4 & 53.0 & 82.7 & 21.1 & 22.0 & 13.7 & 19.1 & 49.5 \\
    & & MI-PIA (Ours) & \textbf{100} & 98.6 & 78.2 & 96.3 & 34.5 & 35.8 & 26.1 & 36.4 & 63.2 \\
    \cmidrule(lr){3-12}
    & & DI-TI-SIA  & \textbf{100} & \textbf{99.9} & 97.0 & 99.7 & 59.2 & 81.2 & 49.1 & 69.7 & 82.0 \\
    & & DI-TI-BSR  & \textbf{100} & \textbf{99.9} & 96.3 & \textbf{99.8} & 58.8 & 78.3 & 48.2 & 68.2 & 81.2 \\
    & & DI-TI-SID  & \textbf{100} & 99.8 & 98.0 & \textbf{99.8} & 66.6 & 84.8 & 57.2 & 76.9 & 85.4    \\
    \cmidrule(lr){3-12}
    && DI-TI-PIA (Ours) & \textbf{100} & \textbf{99.9} & \textbf{98.6} & 99.7 & 68.4 & 84.3 & 59.1 & 77.0 & 85.9 \\
    & & DI-TI-PIA-Mix (Ours) & \textbf{100} & \textbf{99.9} & 98.5 & \textbf{99.8} & \textbf{72.7} & \textbf{87.6} & \textbf{64.9} & \textbf{79.6} & \textbf{87.9} \\

    \cmidrule(lr){2-12}
    & \multirow{4}{*}{\shortstack[c]{Search-based\\Transformation\\Combination}} 
     & L2T & \textbf{100} & \textbf{99.9} & 98.9 & \textbf{99.8} & 73.6 & 89.3 & 76.5 & 88.8 & 90.9 \\
    & & L2T + PIA (Ours) & \textbf{100} & \textbf{99.9} & 99.1 & 99.7 & 77.1 & 90.7 & 79.2 & 91.1 & 92.1 \\
    \cmidrule(lr){3-12}
    & & OPS & \textbf{100} & 99.8 & 99.3 & \textbf{99.8} & 79.5 & 92.2 & 82.7 & 92.2 & 93.2 \\
    & & OPS + PIA (Ours) & \textbf{100} & \textbf{99.9} & \textbf{99.5} & \textbf{99.8} & \textbf{82.3} & \textbf{94.4} & \textbf{83.6} & \textbf{93.8} & \textbf{94.2} \\

    \midrule
    \multirow{12}{*}{DN-121} &
    \multirow{7}{*}{\shortstack[c]{Fixed\\Transformation\\Combination}} 
    &  MI & 86.8 & \textbf{100} & 55.7 & 83.2 & 25.4 & 27.2 & 15.1 & 24.3 & 52.2 \\
    & & MI-PIA (Ours) & 97.6 & \textbf{100} & 79.1 & 95.3 & 36.6 & 43.9 & 26.9 & 38.8 & 64.8 \\
    \cmidrule(lr){3-12}
    & & DI-TI-SIA  & 99.6 & \textbf{100} & 97.2 & 99.1 & 60.3 & 81.3 & 51.7 & 73.6 & 82.9 \\
    & & DI-TI-BSR  & 99.5 & \textbf{100} & 95.2 & 98.5 & 56.5 & 75.0 & 46.0 & 68.4 & 79.9 \\
    & & DI-TI-SID & 99.7 & \textbf{100} & 98.0 & 99.0 & 65.6 & 83.0 & 58.1 & 76.6 & 85.0 \\
    \cmidrule(lr){3-12}
    & & DI-TI-PIA (Ours)  & \textbf{99.8} & \textbf{100} & \textbf{98.4} & 99.2 & 67.2 & 83.3 & 60.6 & 76.1 & 85.6 \\
    & & DI-TI-PIA-Mix (Ours) & 99.6 & \textbf{100} & \textbf{98.4} & \textbf{99.3} & \textbf{70.1} & \textbf{86.5} & \textbf{62.4} & \textbf{80.5} & \textbf{87.1} \\

    \cmidrule(lr){2-12}
    & \multirow{4}{*}{\shortstack[c]{Search-based\\Transformation\\Combination}} 
     & L2T & \textbf{99.9} & \textbf{100} & 98.7 & \textbf{99.9} & 73.1 & 88.4 & 75.6 & 88.1 & 90.5 \\
    & & L2T + PIA (Ours) & \textbf{99.9} & \textbf{100} & 99.0 & \textbf{99.9} & 75.0 & 90.7 & 78.5 & 89.0 & 91.5 \\
    \cmidrule(lr){3-12}
    & & OPS & \textbf{99.9} & \textbf{100} & \textbf{99.2} & \textbf{99.9} & 78.2 & 92.4 & 81.6 & 91.8 & 92.9 \\
    & & OPS + PIA (Ours) & \textbf{99.9} & \textbf{100} & \textbf{99.2} & \textbf{99.9} & \textbf{81.3} & \textbf{93.2} & \textbf{83.4} & \textbf{93.0} & \textbf{93.7} \\

    \bottomrule[0.15em]

    \end{tabular}%
    }
    \caption{Attack success rates (\%) of combining PIA with existing methods under both fixed transformation (DI-TI) and search-based transformation (L2T, OPS) frameworks against eight target models on the ImageNet-compatible dataset.
    The best results within each subgroup are shown in bold.}
    \label{tab:combine}
\end{table*}

% -------------------- Table 3 : ensemble --------------------
\begin{table*}[htbp]
    \centering
    \resizebox{0.68\textwidth}{!}{%
    \begin{tabular}{lccccccccc}
    \toprule[0.15em]
    Attack & 
    RN-50* & DN-121*  & Inc-v3* & MB-v2* & ViT & LeViT & ConViT & PiT & Average
    \\ \midrule
     DI         & \textbf{100}* & \textbf{100}* & 99.6* & \textbf{100}* & 41.5 & 71.6 & 40.6 & 62.1 & 76.9 \\
     TI         & \textbf{100}* & \textbf{100}* & 99.7* & \textbf{100}* & 42.0 & 59.6 & 38.1 & 53.4 & 74.1 \\
     SIA        & \textbf{100}* & \textbf{100}* & \textbf{100}* &  \textbf{100}* & 68.7 & 94.3 & 69.6 & 89.1 & 90.2 \\
     BSR        & \textbf{100}* & \textbf{100}* & 99.9* & \textbf{100}* & 68.3 & 93.6 & 67.8 & 85.6 & 89.4 \\
     SID        & \textbf{100}* & \textbf{100}* & \textbf{100}* & \textbf{100}* & 77.6 & 95.7 & 76.1 & 92.9 & 92.8 \\
     \cmidrule(lr){1-10}
     PIA (Ours) & \textbf{100}* & \textbf{100}* & \textbf{100}* & \textbf{100}* & 78.8 & 95.5 & 77.7 & 93.4 & 93.2 \\
     PIA-Mix (Ours) & \textbf{100}* & \textbf{100}* & \textbf{100}* & \textbf{100}* & \textbf{80.1} & \textbf{97.3} & \textbf{80.9} & \textbf{94.1} & \textbf{94.1} \\

    \bottomrule[0.15em]
    \end{tabular}%
    }
    \caption{Attack success rates (\%) using an ensemble of models as the surrogate model on the ImageNet-compatible dataset.
    All methods are based on MI.
    * indicates the surrogate models.
    The best results are shown in bold.
    }
    \label{tab:ensemble}
\end{table*}

\subsection{Combining with Other Attack Methods}
\label{sec:exp_combine}

Table~\ref{tab:combine} further evaluates the compatibility of PIA with existing attack frameworks.
We consider two categories of combination strategies: fixed transformation combinations and search-based transformation combinations.

\textbf{Fixed Transformation Combination.}
We combine PIA and PIA-Mix with the widely used DI-TI framework.
As shown in Table~\ref{tab:combine}, DI-TI-PIA-Mix achieves the highest average ASR among all fixed transformation combinations across both surrogate settings.
For example, with RN-50 as the surrogate model, DI-TI-PIA-Mix achieves an average ASR of $87.9\%$, outperforming DI-TI-SID ($85.4\%$) by $2.5\%$ and DI-TI-SIA ($82.0\%$) by $5.9\%$.
With DN-121 as the surrogate model, DI-TI-PIA-Mix achieves $87.1\%$, surpassing DI-TI-SID ($85.0\%$) by $2.1\%$.
These results confirm that PIA is compatible with existing attack methods, and can be integrated into common attack pipelines for further improvement.

\textbf{Search-based Transformation Combination.}
We also evaluate PIA within search-based transformation frameworks, including L2T~\cite{zhu2024learning} and OPS~\cite{guo2025boosting}.
Unlike fixed combinations where the transformation pipeline is pre-defined, L2T and OPS search for optimal transformation combinations from a group of candidate transformations during the attack process.
This search procedure incurs substantially higher computational costs, typically over $5\times$ that of fixed combination methods.
Note that PIA, as a fixed transformation method, is fundamentally different from these search-based approaches.
To evaluate whether our perspective transformation can provide complementary diversity even within these sophisticated frameworks, we add PIA to their candidate pools (denoted as L2T + PIA and OPS + PIA) without modifying other configurations.
As shown in the lower part of Table~\ref{tab:combine}, incorporating PIA consistently improves both L2T and OPS across both surrogate settings.
With RN-50 as the surrogate, L2T + PIA achieves $92.1\%$ compared to $90.9\%$ for L2T ($+1.2\%$), and OPS + PIA achieves $94.2\%$ compared to $93.2\%$ for OPS ($+1.0\%$).
Similar improvements are observed with DN-121 ($91.5\%$ vs.\ $90.5\%$ for L2T, $93.7\%$ vs.\ $92.9\%$ for OPS).
This demonstrates that perspective transformation introduces diversity that is complementary to existing transformations, and PIA can serve as a plug-in component to enhance both fixed and search-based attack frameworks without significant additional computational cost (see Table~\ref{tab:time}).

\textbf{Combining with Ensemble-based Attack Framework.}
Table~\ref{tab:ensemble} shows the attack success rates using an ensemble-based attack framework~\cite{LiuCLS17}, where we use the ensemble of RN-50, DN-121, Inc-v3, and MB-v2 as the surrogate models.
Our PIA-Mix achieves the highest average ASR of $94.1\%$ against all eight target models, outperforming SID ($92.8\%$) by $1.3\%$.
PIA alone also surpasses all existing methods with an average ASR of $93.2\%$.
These results further confirm the effectiveness and generality of our methods across different attack frameworks.

\textbf{Computational Efficiency.}
Table~\ref{tab:time} reports the average time for generating one adversarial example on the ImageNet-compatible dataset using an RTX 3090 GPU.
Among fixed transformation methods, PIA requires only $0.739$ seconds per image, which is comparable to BSR ($0.729$s), faster than SIA ($0.823$s), and notably more efficient than SID ($1.217$s).
PIA-Mix requires $0.962$ seconds, still faster than SID.
For search-based methods, L2T and OPS require $8.127$s and $6.311$s per image respectively, as they need to search over candidate transformations during the attack.
Adding PIA to their candidate pools introduces negligible overhead (L2T + PIA: $8.184$s, OPS + PIA: $6.373$s), confirming that PIA can be integrated as a plug-in component with only marginal additional cost.

% -------------------- Table : time --------------------
\begin{table}[htbp]
    \centering
    \resizebox{0.9\linewidth}{!}{%
    \begin{tabular}{llccc}
    \toprule[0.15em]

    \multirow{2}{*}{Category} &
    \multirow{2}{*}{Attack} &
    \multirow{2}{*}{
        \begin{tabular}[c]{@{}c@{}}     
    Copy\\ Number
    \end{tabular}
    } &
    \multirow{2}{*}{
        \begin{tabular}[c]{@{}c@{}}     
    Iteration\\ Number
    \end{tabular}
    } &
    \multirow{2}{*}{
        \begin{tabular}[c]{@{}c@{}}     
    Average\\ Time (s)
    \end{tabular}
    } 
    \\ \\

    \midrule

    \multirow{5}{*}{\shortstack[c]{Fixed\\Transformation}} 
     & SIA      & 20 & 10 & 0.823  \\
     & BSR      & 20 & 10 & 0.729  \\
     & SID      & 20 & 10 & 1.217  \\
    \cmidrule(lr){2-5}
     & PIA (Ours) & 20 & 10 & 0.739  \\
     & PIA-Mix (Ours)  & 20 & 10 & 0.962    \\
     \midrule
     \multirow{4}{*}{\shortstack[c]{Search-based\\Transformation\\Combination}}
     & L2T & 225 & 10 & 8.127 \\
     & L2T + PIA (Ours) & 225 & 10 & 8.184\\
     \cmidrule(lr){2-5}
     & OPS & 200 & 10 & 6.311 \\
     & OPS + PIA (Ours) & 200 & 10 & 6.373 \\

    \bottomrule[0.15em]
    \end{tabular}%
    }
    \caption{
        Average time (s) for each attack to generate an adversarial example on the ImageNet dataset.
        The surrogate model is RN-50.
        Note that for each attack iteration in L2T, the copy number is variable and capped at 225.
    }
    \label{tab:time}
\end{table}

% -------------------- Table 4 : defense --------------------
\begin{table*}[!t]
    \centering
    \resizebox{0.9\textwidth}{!}{%
    \begin{tabular}{clccccccccccccc}
    \toprule[0.15em]

    \multirow{2}{*}{Source}
    & \multirow{2}{*}{Attack} 
    & \multirow{2}{*}{HGD} 
    & \multirow{2}{*}{AT} 
    & \multirow{2}{*}{RS} 
    & \multicolumn{3}{c}{Protected Model: ViT} 
    & \multicolumn{3}{c}{Protected Model: ConViT} 
    & \multirow{2}{*}{Inc-v3$_{\text {ens3}}$} 
    & \multirow{2}{*}{Inc-v3$_{\text {ens4}}$}
    & \multirow{2}{*}{IR-v2$_{\text {ens}}$}
    & \multirow{2}{*}{Average}
    \\
    \cmidrule(lr){6-8}
    \cmidrule(lr){9-11}
    
    &&&&
    & HGD & NRP & DiffPure
    & HGD & NRP & DiffPure
    \\

    \midrule
    \multirow{5}{*}{RN-50} 
    & SIA          & 63.5 & 41.0 & 37.2 & 29.6 & 36.9 & 24.5 & 25.1 & 25.6 & 19.9 & 46.0 & 42.3 & 23.2 & 34.6 \\
    & BSR          & 63.7 & 40.9 & 36.0 & 34.0 & 37.3 & 26.0 & 29.1 & 27.9 & 22.9 & 50.4 & 45.3 & 27.9 & 36.8 \\
    & SID          & 73.7 & 42.2 & 41.6 & 34.3 & 42.3 & 31.0 & 33.7 & 32.0 & 28.4 & 58.1 & 54.5 & 33.2 & 42.1 \\
    \cmidrule(lr){2-15}
    & PIA (Ours) & 74.2 & 42.4 & 40.1 & 36.8 & 42.5 & 30.2 & 35.9 & 34.1 & 27.7 & 59.0 & 55.8 & 35.8 & 42.9 \\
    & PIA-Mix (Ours) & \textbf{76.4} & \textbf{43.5} & \textbf{42.2} & \textbf{38.8} & \textbf{43.0} & \textbf{32.9} & \textbf{38.2} & \textbf{35.8} & \textbf{29.9} & \textbf{65.3} & \textbf{59.8} & \textbf{39.4} & \textbf{45.4} \\

    \midrule
    \multirow{5}{*}{DN-121} 

    & SIA          & 55.9 & 40.7 & 36.0 & 31.5 & 37.1 & 26.3 & 26.2 & 26.1 & 22.5 & 47.3 & 43.6 & 24.2 & 34.8  \\
    & BSR          & 58.9 & 40.5 & 35.9 & 35.3 & 36.7 & 26.7 & 31.9 & 25.5 & 22.8 & 50.7 & 46.5 & 28.0 & 36.6 \\
    & SID          & 71.1 & 41.0 & 41.5 & 32.5 & 42.8 & 32.9 & 32.8 & 33.1 & 29.6 & 60.8 & 56.5 & 37.1 & 42.6 \\
    \cmidrule(lr){2-15}
    & PIA (Ours) & 70.2 & 41.8 & 40.7 & 35.8 & \textbf{43.3} & \textbf{33.8} & 34.5 & 33.3 & 30.4 & 61.4 & 56.2 & 38.1 & 43.3 \\
    & PIA-Mix (Ours) & \textbf{73.8} & \textbf{43.1} & \textbf{42.5} & \textbf{41.2} & 42.8 & 33.5 & \textbf{38.3} & \textbf{33.7} & \textbf{30.9} & \textbf{66.9} & \textbf{61.1} & \textbf{40.6} & \textbf{45.7} \\

    \bottomrule[0.15em]
    \end{tabular}%
    }
    \caption{Attack success rates (\%) against various defense methods on the ImageNet-compatible dataset.
    All methods are based on MI.
    The best results are shown in bold.
    }
    \label{tab:defense}
\end{table*}

\subsection{Evaluation under Defense Mechanisms}
\label{sec:exp_defense}

Table~\ref{tab:defense} shows the attack success rates against various defense mechanisms on the ImageNet-compatible dataset.
We evaluate on a comprehensive set of 12 defense configurations, including three standalone defenses (HGD~\cite{liao2018defense}, AT~\cite{MadryMSTV18}, RS~\cite{cohen2019certified}), purification-based defenses (HGD, NRP~\cite{naseer2020self}, DiffPure~\cite{nie2022DiffPure}) applied to protect ViT and ConViT (the most difficult models to attack in Table~\ref{tab:sole}), and three ensemble-based adversarially trained models (Inc-v3$_{\text{ens3}}$, Inc-v3$_{\text{ens4}}$, IR-v2$_{\text{ens}}$).

Our PIA-Mix achieves the highest average ASR across all defense configurations in both surrogate settings.
With RN-50 as the surrogate model, PIA-Mix achieves an average ASR of $45.4\%$, outperforming SID ($42.1\%$) by $3.3\%$ and BSR ($36.8\%$) by $8.6\%$.
With DN-121 as the surrogate model, PIA-Mix achieves $45.7\%$, surpassing SID ($42.6\%$) by $3.1\%$.
PIA alone also outperforms all existing methods (e.g., $42.9\%$ vs.\ $42.1\%$ for SID on RN-50), further confirming the effectiveness of perspective transformation even under strong defense mechanisms.

\subsection{Generalization and Additional Analysis}
\label{sec:exp_generalization}

\textbf{Results on CIFAR-10.}
Table~\ref{tab:cifar} shows the attack success rates on the CIFAR-10 dataset, evaluating whether our methods generalize beyond ImageNet.
Notably, all hyperparameters of our methods, including the maximum distortion budget $\delta_{max}$, are kept exactly the same as in the ImageNet experiments.
We use three surrogate settings: Inc-v3, MB-v2, and their ensemble (Ens), and evaluate on four target models, including two adversarially trained models (RN-20-ens-gal and RN-20-ens-dv).
Our PIA and PIA-Mix consistently achieve the highest average ASRs across all three surrogate settings.
With Inc-v3 as the surrogate, PIA-Mix achieves an average ASR of $87.3\%$, outperforming SIA ($82.9\%$) by $4.4\%$ and SID ($81.9\%$) by $5.4\%$.
With the ensemble surrogate, PIA-Mix further achieves $91.2\%$, surpassing both SIA and SID ($87.1\%$) by $4.1\%$.
Notably, our methods demonstrate particularly strong advantages against adversarially trained models.
With Inc-v3 as the surrogate, PIA-Mix achieves ASRs of $74.6\%$ and $75.0\%$ against RN-20-ens-gal and RN-20-ens-dv, surpassing SIA ($65.6\%$ and $66.4\%$) by $9.0\%$ and $8.6\%$, respectively.
These results confirm that our methods generalize well across different datasets.

%
% -------------------- Table : CIFAR 10 --------------------
\begin{table}[!t]
    \centering
    \resizebox{0.49\textwidth}{!}{%
    \begin{tabular}{clccccc}
    \toprule[0.15em]

    \multirow{2}{*}{Source} & 
    \multirow{2}{*}{Attack} 
    & \multirow{2}{*}{
        \begin{tabular}[c]{@{}c@{}}
        Inc-v3
        \end{tabular}
        } 
    & \multirow{2}{*}{
        \begin{tabular}[c]{@{}c@{}}
         MB-v2
        \end{tabular}
        } 
    & \multirow{2}{*}{
        \begin{tabular}[c]{@{}c@{}}
        RN-20-\\ens-gal
        \end{tabular}
        } 
    & \multirow{2}{*}{
        \begin{tabular}[c]{@{}c@{}}
        RN-20-\\ens-dv
        \end{tabular}
        } 
    & \multirow{2}{*}{
        \begin{tabular}[c]{@{}c@{}}
        Avg
        \end{tabular}
        } 
    \\ \\

    \midrule
    \multirow{8}{*}{Inc-v3} 

    & DI         & 88.8 & 78.5 & 34.4 & 33.5 & 58.8 \\
    & TI         & 86.0 & 77.3 & 37.6 & 35.7 & 59.2 \\
    & SIA        & \textbf{100} & 99.6 & 65.6 & 66.4 & 82.9 \\
    & BSR        & \textbf{100} & 96.0 & 63.2 & 62.6 & 80.5 \\
    & SID & \textbf{100} & 97.0 & 65.0 & 65.7 & 81.9 \\
     \cmidrule(lr){2-7}
    & PIA (Ours) & \textbf{100} & 98.6 & 69.6 & 70.9 & 84.8 \\
    & PIA-Mix (Ours) & \textbf{100} & \textbf{99.7} & \textbf{74.6} & \textbf{75.0} & \textbf{87.3} \\
    
    \midrule
    \multirow{8}{*}{MB-v2} 

    & DI         & 99.4 & \textbf{100} & 46.7 & 50.1 & 74.1 \\
    & TI         & 97.9 & \textbf{100} & 49.4 & 49.7 & 74.2 \\
    & SIA        & \textbf{100} &  \textbf{100} & 60.2 & 63.4 & 80.9 \\
    & BSR        & 99.5 & \textbf{100} & 58.6 & 59.5 & 79.4 \\
    & SID & 99.7 & \textbf{100} & 61.8 & 63.4 & 81.2 \\
    \cmidrule(lr){2-7}
    & PIA (Ours) & 99.7 & \textbf{100} & 64.5 & 68.5 & 83.2 \\
    & PIA-Mix (Ours) & \textbf{100} & \textbf{100} & \textbf{69.0} & \textbf{70.8} & \textbf{85.0} \\

     \midrule
     \multirow{8}{*}{Ens} 

     & DI         & 96.8 & 99.2 & 48.1 & 52.5 & 74.2 \\
     & TI         & 88.3 & 95.6 & 50.1 & 50.2 & 71.0 \\
     & SIA        & \textbf{100} & \textbf{100} & 73.6 & 74.8 & 87.1 \\
     & BSR        & \textbf{100} & \textbf{100} & 69.3 & 74.3 & 85.9 \\
     & SID & \textbf{100} & \textbf{100} & 73.1 & 75.2 & 87.1 \\
     \cmidrule(lr){2-7}
    & PIA (Ours) & \textbf{100} & \textbf{100} & 77.3 & 77.6 & 88.7 \\
    & PIA-Mix (Ours) & \textbf{100} & \textbf{100} & \textbf{83.6} & \textbf{81.1} & \textbf{91.2} \\ 

    \bottomrule[0.15em]
    \end{tabular}%
    }
    \caption{Attack success rates (\%) against four target models on the CIFAR-10 dataset.
    All methods are based on MI.
    The best results are shown in bold.
    }
    \label{tab:cifar}
\end{table}

\textbf{Attacking Multimodal LLMs.}
Table~\ref{tab:MLLMs} evaluates the attack effectiveness against multimodal large language models (MLLMs).
We randomly sample 100 images from the ImageNet-compatible dataset and use ResNet-50 as the surrogate model.
We evaluate three commercial MLLMs: Qwen2-VL, Claude-3.5, and GPT-4o.
To comprehensively assess adversarial transferability, we adopt two evaluation protocols.
The first is a \textit{closed-ended} protocol, where we prompt the MLLM with ``Is this image a photo of \{true label\}? Yes or No?'' and consider a ``No'' response as a successful attack.
The second is an \textit{open-ended} protocol designed to address the concern that adversarial examples may only be effective under constrained binary questions.
Specifically, we prompt the MLLM with ``What is the main object in this image?'' without revealing the true label, and employ a separate LLM judge (Claude-3.5) to determine whether the model's free-form response correctly identifies the true label or any of its synonyms; if not, the attack is considered successful.

As shown in Table~\ref{tab:MLLMs}, our methods achieve the highest ASRs under both evaluation protocols.
Under the basic MI-based framework, PIA-Mix achieves an average ASR of $66.0\%$ / $75.0\%$ (closed / open), outperforming SID ($59.3\%$ / $66.0\%$) by $6.7\%$ / $9.0\%$.
When combined with DI-TI, PIA-Mix further improves to $69.6\%$ / $79.3\%$, surpassing DI-TI-SID ($64.6\%$ / $72.3\%$) by $5.0\%$ / $7.0\%$.
Importantly, our methods maintain clear advantages under the open-ended protocol, where the MLLM must freely describe the image content without any label hint.
This demonstrates that the adversarial transferability of PIA is not limited to constrained binary questions, but extends to more general and challenging evaluation settings.

% - - - - - - - - - - exp on multimodal LLMs
\begin{table}[htbp]
    \centering
    \resizebox{0.95\linewidth}{!}{%
    \begin{tabular}{lcccc}
    \toprule[0.15em]
    Attack & Qwen2-VL & Claude-3.5 & GPT-4o & Average \\
    \midrule
    SIA            & 83 / 70 & 49 / 68 & 40 / 54 & 57.3 / 64.0 \\
    BSR            & 80 / 69 & 49 / 66 & 38 / 51 & 55.6 / 62.0 \\
    SID            & 85 / 72 & 54 / 72 & 39 / 54 & 59.3 / 66.0 \\
    PIA (Ours)     & 87 / 73 & 56 / 76 & 43 / 58 & 62.0 / 69.0 \\
    PIA-Mix (Ours) & \textbf{89} / \textbf{79} & \textbf{59} / \textbf{79} & \textbf{50} / \textbf{67} & \textbf{66.0} / \textbf{75.0} \\
    \midrule
    DI-TI-SIA            & 86 / 73 & 52 / 71 & 51 / 67 & 63.0 / 70.3 \\
    DI-TI-BSR            & 81 / 71 & 49 / 67 & 47 / 64 & 59.0 / 67.3 \\
    DI-TI-SID            & 86 / 74 & 57 / 77 & 51 / 66 & 64.6 / 72.3 \\
    DI-TI-PIA (Ours)     & 89 / 78 & 58 / 77 & 53 / 70 & 66.6 / 75.0 \\
    DI-TI-PIA-Mix (Ours) & \textbf{91} / \textbf{81} & \textbf{61} / \textbf{82} & \textbf{57} / \textbf{75} & \textbf{69.6} / \textbf{79.3} \\
    \bottomrule[0.15em]
    \end{tabular}%
    }
    \caption{Attack success rates (\%) against multimodal LLMs. 
    Results are reported as Closed / Open-ended, corresponding to two evaluation protocols: a closed-ended binary prompt (``Is this image a photo of \{true label\}? Yes or No?'') and an open-ended descriptive prompt (``What is the main object in this image?'').
    We randomly sample 100 images from the ImageNet-compatible dataset and use ResNet-50 as the surrogate model.
    The best results are shown in bold.
    }
    \label{tab:MLLMs}
  \end{table}

\textbf{Stability Analysis.}
As our methods involve random sampling of four corner points, there might be concerns about the statistical robustness of the reported single-run results.
It is important to note that other recent state-of-the-art methods also rely heavily on randomness during their transformations. For instance, SIA randomly selects block positions and transformation operations, BSR randomizes block positions and rotation angles, and SID applies random positional perturbations.
To ensure a rigorous comparison and to verify that our improvements are consistent, we report the mean and standard deviation of the attack success rates over 5 independent runs in Table~\ref{tab:stability}.
As shown, both PIA and PIA-Mix exhibit very low variance across all target models; in particular, the maximum standard deviation of PIA-Mix is only $0.30\%$, which is smaller than the corresponding maxima of BSR ($0.81\%$) and SID ($0.64\%$).
These results confirm that the performance of our methods is highly stable and the reported improvements are robust.

\begin{table*}[htbp]
    \centering
    \resizebox{0.92\textwidth}{!}{%
    \begin{tabular}{lc|cccccccccc}
    \toprule[0.15em]
    Method & Runs & 
    RN-50 & DN-121 & Inc-v3 & MB-v2 & ViT & LeViT & ConViT & PiT 
    \\ \midrule
    SIA & 5 & 100.0$\pm$0.00 & 99.9$\pm$0.04 & 95.3$\pm$0.50 & 99.6$\pm$0.08 & 45.9$\pm$0.48 & 69.3$\pm$0.38 & 38.6$\pm$0.51 & 58.3$\pm$0.47 \\
    BSR & 5 & 100.0$\pm$0.00 & 99.8$\pm$0.12 & 95.0$\pm$0.19 & 99.7$\pm$0.10 & 47.6$\pm$0.81 &70.0$\pm$0.73 & 40.1$\pm$0.74 & 58.5$\pm$0.78 \\
    SID & 5 & 100.0$\pm$0.00 & 99.9$\pm$0.06 & 97.6$\pm$0.18 & 99.6$\pm$0.09 & 57.9$\pm$0.44 & 81.6$\pm$0.38 & 52.4$\pm$0.47 & 66.7$\pm$0.64 \\
    \midrule
    \rowcolor{gray!10}
    PIA (Ours) & 5 & 100.0$\pm$0.00 & 99.9$\pm$0.08 & 97.8$\pm$0.19 & 99.7$\pm$0.12 & 60.3$\pm$0.39 & 82.2$\pm$0.50 & 54.5$\pm$0.42 & 69.0$\pm$0.35 \\
    \rowcolor{gray!10}
    PIA-Mix (Ours) & 5 & 100.0$\pm$0.00 & 99.9$\pm$0.06 & 98.5$\pm$0.21 & 99.8$\pm$0.10 & 64.3$\pm$0.28 & 87.4$\pm$0.30 & 58.1$\pm$0.20 & 76.3$\pm$0.14 \\
    \bottomrule[0.15em]

    \end{tabular}%
    }
    \caption{Mean and standard deviation of attack success rates (\%) over 5 independent runs on the ImageNet-compatible dataset. The surrogate model is RN-50.}

    \label{tab:stability}
\end{table*}

\section{Ablation Study}
\label{sec:ablation}

We conduct systematic ablation studies to validate the contribution of each component in our method.
All experiments use RN-50 as the surrogate model and are evaluated on the ImageNet-compatible dataset.
Sections~\ref{sec:abla_vertex} and \ref{sec:abla_distortion} study PIA in isolation (without the complementary transformation pool), directly demonstrating the effectiveness of our perspective transformation itself.
Section~\ref{sec:abla_pool} then investigates the pool composition of PIA-Mix.

\begin{table*}[!t]
    \centering
    \resizebox{0.92\textwidth}{!}{%
    \begin{tabular}{ccc|ccccccccc}
    \toprule[0.15em]
    Method &
    Degrees of Freedom &
    Vertex Mode Distribution &
    RN-50 & DN-121 & Inc-v3 & MB-v2 & ViT & LeViT & ConViT & PiT & Average
    \\ \midrule
    \multirow{8}{*}{PIA}
    & 8       & (1,0,0,0)             & 100 & 99.9 & 98.0 & 99.7 & 55.4 & 79.0 & 49.3 & 64.2 & 80.7 \\
    & $8\!\to\!6$ & (1,0,0,0), $h_{31}{=}h_{32}{=}0$ & 100 & 99.9 & 97.1 & 99.4 & 47.6 & 71.8 & 41.7 & 58.2 & 77.0 \\ 
    & 6       & (0,1,0,0)             & 100 & 99.9 & 97.3 & 99.3 & 47.3 & 71.9 & 41.3 & 57.9 & 76.9 \\
    & 4       & (0,0,1,0)             & 100 & 99.4 & 95.2 & 97.9 & 45.7 & 69.1 & 37.8 & 56.8 & 75.2 \\
    & 2       & (0,0,0,1)             & 100 & 99.0 & 92.8 & 96.1 & 43.0 & 66.6 & 36.4 & 55.2 & 73.6 \\
    & 8/6/4/2 & (0.25,0.25,0.25,0.25) & 100 & 99.9 & 98.0 & 99.8 & 58.7 & 80.6 & 52.3 & 67.0 & 82.0 \\
    & 8/6/4/2 & (0.1,0.2,0.3,0.4)     & 100 & 99.9 & 97.8 & 99.6 & 55.2 & 78.4 & 48.6 & 63.1 & 80.3 \\
    & 8/6/4/2 & (0.4,0.3,0.2,0.1)     & \textbf{100} & \textbf{99.9} & \textbf{98.1} & \textbf{99.8} & \textbf{60.0} & \textbf{81.9} & \textbf{54.3} & \textbf{69.3} & \textbf{82.9} \\

    \bottomrule[0.15em]
    \end{tabular}%
    }
    \caption{Ablation study on the degrees of freedom (DOF) and multi-DOF vertex sampling strategy.
    All experiments use PIA without the complementary transformation pool, with RN-50 as the surrogate model on the ImageNet-compatible dataset.
    The vertex mode distribution $(\pi_4, \pi_3, \pi_2, \pi_1)$ specifies the sampling probabilities for 8, 6, 4, and 2 DOF modes, respectively.
    The $h_{31}{=}h_{32}{=}0$ row forces the projective terms of the 8-DOF sampling to zero, reducing it to a 6-DOF affine map, as a direct ablation of the higher-order terms.
    The best results are shown in bold.}
    \label{tab:abla_vertex}
\end{table*}

\begin{table*}[!t]
    \centering
    \resizebox{0.68\textwidth}{!}{%
    \begin{tabular}{cc|ccccccccc}
    \toprule[0.15em]
    Method &
    Distortion Rate &
    RN-50 & DN-121 & Inc-v3 & MB-v2 & ViT & LeViT & ConViT & PiT & Average
    \\ \midrule
    \multirow{9}{*}{PIA} 
    & 0.1 & 100 & 99.5 & 96.1 & 99.3 & 46.9 & 68.2 & 41.0 & 58.6 & 76.2 \\
    & 0.2 & 100 & 99.7 & 96.5 & 99.5 & 49.2 & 70.7 & 42.5 & 61.1 & 77.4 \\
    & 0.3 & 100 & 99.7 & 97.0 & 99.6 & 53.7 & 75.9 & 48.0 & 63.1 & 79.6 \\
    & 0.4 & 100 & 99.9 & 97.3 & 99.7 & 56.6 & 77.2 & 50.8 & 65.2 & 80.8 \\
    & 0.5 & 100 & 99.9 & 97.9 & 99.8 & 58.8 & 79.5 & 52.0 & 66.4 & 81.8 \\
    & 0.6 & \textbf{100} & \textbf{99.9} & 98.1 & \textbf{99.8} & \textbf{60.0} & \textbf{81.9} & \textbf{54.3} & \textbf{69.3} & \textbf{82.9} \\
    & 0.7 & 100 & 99.9 & \textbf{98.2} & 99.7 & 59.1 & 80.5 & 52.5 & 67.1 & 82.1 \\
    & 0.8 & 100 & 99.6 & 97.5 & 99.7 & 56.7 & 78.3 & 51.8 & 63.4 & 80.9 \\
    & 0.9 & 100 & 99.5 & 96.8 & 98.3 & 49.8 & 74.4 & 44.8 & 61.3 & 78.1 \\

    \bottomrule[0.15em]
    \end{tabular}%
    }
    \caption{Ablation study on the distortion rate $\delta_{max}$.
    All experiments use PIA with the multi-DOF vertex sampling distribution $(0.4, 0.3, 0.2, 0.1)$, with RN-50 as the surrogate model on the ImageNet-compatible dataset.
    The best results are shown in bold.}
    \label{tab:abla_distortion}
\end{table*}

\begin{table*}[!t]
    \centering
    \resizebox{0.83\textwidth}{!}{%
    \begin{tabular}{ccc|ccccccccc}
    \toprule[0.15em]
    Method &
    Pool Method &
    Probability List &
    RN-50 & DN-121 & Inc-v3 & MB-v2 & ViT & LeViT & ConViT & PiT & Average
    \\ \midrule
    \multirow{7}{*}{PIA-Mix} 
    & PIA                       & (1.0)        & 100 & 99.9 & 98.1 & 99.8 & 60.0 & 81.9 & 54.3 & 69.3 & 82.9 \\
    & PIA, SIA                  & (0.5, 0.5)   & 100 & 99.9 & 98.3 & 99.8 & 62.1 & 82.5 & 55.2 & 70.8 & 83.6 \\
    & PIA, BSR                  & (0.5, 0.5)   & 100 & 99.8 & 98.4 & 99.6 & 62.6 & 81.7 & 55.1 & 70.2 & 83.4 \\
    & PIA, SID                  & (0.5, 0.5)   & \textbf{100} & \textbf{99.9} & \textbf{98.5} & \textbf{99.8} & \textbf{64.5} & \textbf{86.9} & \textbf{58.5} & 76.2 & \textbf{85.5} \\
    & PIA, BSR, SIA & (0.5, 0.25, 0.25)        & 100 & 99.9 & 98.4 & 99.8 & 62.8 & 84.3 & 55.9 & 73.7 & 84.4 \\
    & SID, BSR, SIA & (0.5, 0.25, 0.25)        & 100 & 99.9 & 98.2 & 99.8 & 61.0 & 82.5 & 54.4 & 72.2 & 83.5 \\
    & PIA, SID, BSR, SIA & (0.5, 1/6, 1/6, 1/6)& 100 & 99.9 & 98.5 & 99.7 & 63.8 & 85.8 & 57.2 & \textbf{76.5} & 85.2 \\

    \bottomrule[0.15em]
    \end{tabular}%
    }
    \caption{Ablation study on the composition of the complementary transformation pool $\mathcal{P}$ in PIA-Mix.
    The probability list specifies the sampling probability of each method in the pool.
    All experiments use RN-50 as the surrogate model on the ImageNet-compatible dataset.
    The best results are shown in bold.}

    \label{tab:abla_pool}
\end{table*}

\subsection{Effect of Multi-DOF Vertex Sampling}
\label{sec:abla_vertex}

Table~\ref{tab:abla_vertex} isolates the effect of our multi-DOF vertex sampling within PIA, without the complementary transformation pool.

We first examine the impact of using a single fixed DOF.
The results show a clear monotonic trend: increasing the DOF consistently improves attack performance, with average ASRs of $73.6\%$, $75.2\%$, $76.9\%$, and $80.7\%$ for 2, 4, 6, and 8 DOF, respectively.
This validates our motivation that higher geometric flexibility produces more diverse transformations and leads to stronger adversarial transferability.
The improvement is particularly pronounced on challenging target models such as ViT and ConViT, where the ASR increases from $43.0\%$ and $36.4\%$ (2-DOF) to $55.4\%$ and $49.3\%$ (8-DOF), respectively.
Moreover, forcing $h_{31}{=}h_{32}{=}0$ (reducing the 8-DOF map to a 6-DOF affine one) lowers the average ASR from $80.7\%$ to $77.0\%$, confirming that the higher-order terms contribute meaningfully under the budget.

We then evaluate our multi-DOF vertex sampling strategy, which mixes transformations of different DOFs according to a categorical distribution $\boldsymbol{\pi} = (\pi_4, \pi_3, \pi_2, \pi_1)$.
Two of the three multi-DOF configurations outperform the best single-DOF setting (8-DOF, $80.7\%$), confirming the benefit of covering the full geometric subspace hierarchy rather than relying on a single DOF level.
Among them, the distribution $(0.4, 0.3, 0.2, 0.1)$, which assigns higher probability to higher-DOF modes, achieves the best average ASR of $82.9\%$, outperforming the single 8-DOF setting by $2.2\%$.
We adopt this distribution as the default for all subsequent experiments.

\subsection{Effect of Distortion Rate}
\label{sec:abla_distortion}

Table~\ref{tab:abla_distortion} examines the effect of the maximum distortion rate $\delta_{max}$, which controls the displacement range of corner positions in Equation~\ref{eq:perspective_transformation_corner}.
We use PIA with the optimal multi-DOF vertex sampling distribution $(0.4, 0.3, 0.2, 0.1)$.

The average ASR increases steadily as $\delta_{max}$ grows from $0.1$ ($76.2\%$) to $0.6$ ($82.9\%$), then decreases for larger values ($82.1\%$ at $0.7$, $78.1\%$ at $0.9$).
This reveals a clear trade-off: a small $\delta_{max}$ limits the diversity of perspective transformations, while an excessively large $\delta_{max}$ introduces heavily distorted views that degrade gradient quality.
The optimal value $\delta_{max} = 0.6$ achieves the best balance between transformation diversity and gradient reliability, and we adopt it as the default setting.

\subsection{Effect of Complementary Transformation Pool}
\label{sec:abla_pool}

Table~\ref{tab:abla_pool} investigates the composition of the complementary transformation pool $\mathcal{P}$ in PIA-Mix.
Importantly, the total number of transformed copies $N$ per iteration is kept strictly constant across all configurations, ensuring that any performance improvement arises purely from transformation diversity rather than additional computation.

We fix PIA's sampling probability at $0.5$ and distribute the remaining probability evenly among auxiliary methods.
We consider three recent input transformation methods as pool candidates: BSR, SIA, and SID.
Adding any single method to PIA's pool consistently improves over PIA alone ($82.9\%$): PIA + SIA achieves $83.6\%$, PIA + BSR achieves $83.4\%$, and PIA + SID achieves $85.5\%$.
Among all configurations, PIA + SID yields the highest average ASR of $85.5\%$, slightly outperforming the full four-method pool (PIA + SID + BSR + SIA, $85.2\%$) and notably surpassing the pool without SID (PIA + BSR + SIA, $84.4\%$).
This suggests that SID provides the strongest complementarity to PIA, likely because its multi-scale spatial transformations with local image fusion are inherently orthogonal to PIA's global geometric transformations.
Adding BSR and SIA on top of PIA + SID does not yield further improvement, indicating diminishing returns from pooling methods that operate in similar (spatial) domains.

Notably, replacing PIA with the combination of SID, BSR, and SIA achieves only $83.5\%$, which is substantially lower than PIA + SID ($85.5\%$) and only marginally exceeds PIA alone ($82.9\%$).
This demonstrates that PIA is the primary driver of the transferability improvement in PIA-Mix, and that simply combining existing methods without PIA cannot replicate its effectiveness.
Based on these results, we adopt PIA + SID with probabilities $(0.5, 0.5)$ as the default PIA-Mix configuration.

\begin{figure*}[htbp]
    \centering
    \includegraphics[width=0.72\textwidth]{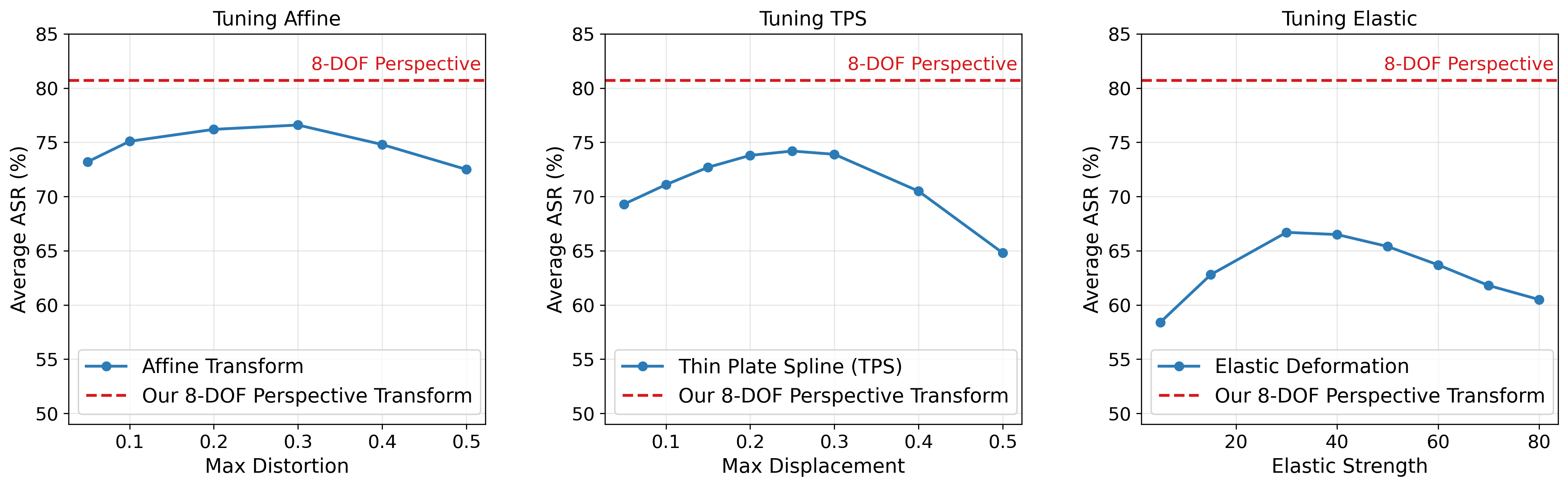}
    \caption{Comparison of our 8-DOF perspective transformation against three alternative transformations with exhaustive hyperparameter tuning, using RN-50 as the surrogate model.
    The red dashed line indicates the average ASR of our 8-DOF perspective transformation ($80.7\%$).
    \textbf{(Left)}~Affine transformation (6-DOF), sweeping the maximum distortion rate.
    \textbf{(Middle)}~Thin-plate spline (TPS) with a $3\times3$ control grid (18-DOF), sweeping the maximum displacement.
    \textbf{(Right)}~Elastic deformation (high-DOF), sweeping the elastic strength.
    Our perspective transformation consistently outperforms all alternatives at their respective optimal settings.}
    \label{fig:compare_non_rigid}
\end{figure*}

\subsection{Comparison with Alternative Transformations}
\label{sec:compare_non_rigid}

As discussed in our motivation (Section~\ref{sec:motivation}), a natural question is whether transformations with different degrees of freedom, such as lower-DOF global or higher-DOF non-rigid operations, could yield better adversarial transferability than our 8-DOF perspective transformation.
To systematically address this, we conduct comprehensive hyperparameter sweeps for three representative alternative transformations and compare their best-case performance against our method.
All experiments use the MI-FGSM framework with RN-50 as the surrogate model, and we report the average ASR across all eight target models.
The red dashed line in Figure~\ref{fig:compare_non_rigid} indicates the ASR of our 8-DOF perspective transformation ($80.7\%$), which uses only a single DOF level without multi-DOF vertex sampling or the complementary transformation pool, serving as a clean baseline.

\textbf{Affine Transformation (6-DOF).}
We parameterize the affine transformation using randomized corner displacements controlled by a maximum distortion rate relative to the image width, and sweep this rate from $0.05$ to $0.5$.
As shown in Figure~\ref{fig:compare_non_rigid}(Left), affine transformation reaches its optimal ASR of approximately $76\%$ at a distortion rate around $0.3$, which is about $4.7\%$ below our perspective baseline.
This confirms that affine's 6-DOF geometric space, while preserving structural integrity, lacks sufficient diversity for strong transferability.

\textbf{Thin-Plate Spline (TPS) Transformation.}
We adopt a $3 \times 3$ grid of control points, where each point is independently displaced within a bounded range.
With 9 control points each having 2D displacements, TPS has 18 effective DOF, which is more than double that of our perspective transformation.
We sweep the maximum displacement from $0.05$ to $0.5$.
As shown in Figure~\ref{fig:compare_non_rigid}(Middle), TPS reaches its optimal ASR below $75\%$ at a displacement around $0.25$, still well below our baseline.
Despite having considerably higher DOF, TPS introduces non-rigid warping that bends straight lines and distorts object structures, producing unreliable adversarial gradients.

\textbf{Elastic Deformation.}
Elastic deformation generates a dense pixel-wise displacement field smoothed by an averaging kernel, representing an extremely high effective DOF.
We sweep the elastic strength from $10$ to $80$.
As shown in Figure~\ref{fig:compare_non_rigid}(Right), elastic deformation achieves its optimal ASR below $70\%$ at a strength around $30$, falling short of our baseline by over $10\%$.
The unconstrained pixel-level warping severely disrupts semantic and geometric structures of the input, leading to highly unreliable gradient directions for transfer attacks.

These results comprehensively validate our transformation choice.
Even after exhaustive hyperparameter tuning, no alternative transformation, whether lower-DOF global (affine) or higher-DOF non-rigid (TPS, elastic), surpasses our 8-DOF perspective transformation.
This confirms the conjecture in Section~\ref{sec:motivation}: the projective transformation achieves an effective balance between geometric diversity and structural preservation.
It is the highest-DOF global transformation that still preserves straight lines and geometric structures, maximizing input diversity without the destructive warping introduced by non-rigid operations.

\section{Conclusion}

In this work, we revisit existing input transformation-based attacks and identify that recent methods rely on spatial augmentations within the image plane, overlooking global perspective transformations that naturally arise from viewpoint changes.
To address this limitation, we propose Perspective-Invariant Attack (PIA), which introduces a multi-DOF vertex sampling strategy to systematically explore the perspective transformation hierarchy and generate geometrically diverse input variations.
We further propose PIA-Mix, which combines PIA with complementary transformations via a configurable transformation pool for additional performance gains. Comprehensive experiments demonstrate that PIA and PIA-Mix significantly outperform prior methods across various models, datasets, and defense mechanisms. Empirical results also show that our methods can be easily and effectively integrated with other attack techniques and frameworks, providing a practical and effective approach for evaluating model robustness.

% \clearpage

%%%%%%%%% REFERENCES
\bibliographystyle{IEEEtran}
\bibliography{sample-base}

\begin{thebibliography}{10}
\providecommand{\url}[1]{#1}
\csname url@samestyle\endcsname
\providecommand{\newblock}{\relax}
\providecommand{\bibinfo}[2]{#2}
\providecommand{\BIBentrySTDinterwordspacing}{\spaceskip=0pt\relax}
\providecommand{\BIBentryALTinterwordstretchfactor}{4}
\providecommand{\BIBentryALTinterwordspacing}{\spaceskip=\fontdimen2\font plus
\BIBentryALTinterwordstretchfactor\fontdimen3\font minus
  \fontdimen4\font\relax}
\providecommand{\BIBforeignlanguage}[2]{{%
\expandafter\ifx\csname l@#1\endcsname\relax
\typeout{** WARNING: IEEEtran.bst: No hyphenation pattern has been}%
\typeout{** loaded for the language `#1'. Using the pattern for}%
\typeout{** the default language instead.}%
\else
\language=\csname l@#1\endcsname
\fi
#2}}
\providecommand{\BIBdecl}{\relax}
\BIBdecl

\bibitem{krizhevsky2012imagenet}
A.~Krizhevsky, I.~Sutskever, and G.~E. Hinton, ``Imagenet classification with
  deep convolutional neural networks,'' \emph{{NeurIPS}}, 2012.

\bibitem{He_2016_CVPR}
K.~He, X.~Zhang, S.~Ren, and J.~Sun, ``Deep residual learning for image
  recognition,'' in \emph{{CVPR}}, 2016.

\bibitem{SzegedyVISW16}
C.~Szegedy, V.~Vanhoucke, S.~Ioffe, J.~Shlens, and Z.~Wojna, ``Rethinking the
  inception architecture for computer vision,'' in \emph{{CVPR}}, 2016.

\bibitem{geiger2012we}
A.~Geiger, P.~Lenz, and R.~Urtasun, ``Are we ready for autonomous driving? the
  kitti vision benchmark suite,'' in \emph{{CVPR}}, 2012.

\bibitem{ronneberger2015u}
O.~Ronneberger, P.~Fischer, and T.~Brox, ``U-net: Convolutional networks for
  biomedical image segmentation,'' in \emph{{MICCAI}}, 2015.

\bibitem{schroff2015facenet}
F.~Schroff, D.~Kalenichenko, and J.~Philbin, ``Facenet: A unified embedding for
  face recognition and clustering,'' in \emph{{CVPR}}, 2015.

\bibitem{GoodfellowSS14}
I.~J. Goodfellow, J.~Shlens, and C.~Szegedy, ``Explaining and harnessing
  adversarial examples,'' in \emph{{ICLR}}, 2015.

\bibitem{KurakinGB17a}
A.~Kurakin, I.~J. Goodfellow, and S.~Bengio, ``Adversarial examples in the
  physical world,'' in \emph{{ICLR} (Workshop)}, 2017.

\bibitem{eykholt2018robust}
K.~Eykholt, I.~Evtimov, E.~Fernandes, B.~Li, A.~Rahmati, C.~Xiao, A.~Prakash,
  T.~Kohno, and D.~Song, ``Robust physical-world attacks on deep learning
  visual classification,'' in \emph{{CVPR}}, 2018.

\bibitem{MadryMSTV18}
A.~Madry, A.~Makelov, L.~Schmidt, D.~Tsipras, and A.~Vladu, ``Towards deep
  learning models resistant to adversarial attacks,'' in \emph{{ICLR}}, 2018.

\bibitem{XieZZBWRY19}
C.~Xie, Z.~Zhang, Y.~Zhou, S.~Bai, J.~Wang, Z.~Ren, and A.~L. Yuille,
  ``Improving transferability of adversarial examples with input diversity,''
  in \emph{{CVPR}}, 2019.

\bibitem{DongLPS0HL18}
Y.~Dong, F.~Liao, T.~Pang, H.~Su, J.~Zhu, X.~Hu, and J.~Li, ``Boosting
  adversarial attacks with momentum,'' in \emph{{CVPR}}, 2018.

\bibitem{Wang_2021_admix}
X.~Wang, X.~He, J.~Wang, and K.~He, ``Admix: Enhancing the transferability of
  adversarial attacks,'' in \emph{{ICCV}}, 2021.

\bibitem{moosavi2016deepfool}
S.-M. Moosavi-Dezfooli, A.~Fawzi, and P.~Frossard, ``Deepfool: a simple and
  accurate method to fool deep neural networks,'' in \emph{{CVPR}}, 2016.

\bibitem{carlini2017towards}
N.~Carlini and D.~Wagner, ``Towards evaluating the robustness of neural
  networks,'' in \emph{IEEE Symposium on Security and Privacy ({SP})}, 2017,
  pp. 39--57.

\bibitem{BrendelRB18_2018_iclr}
W.~Brendel, J.~Rauber, and M.~Bethge, ``Decision-based adversarial attacks:
  Reliable attacks against black-box machine learning models,'' in
  \emph{{ICLR}}, 2018.

\bibitem{LinS00H20}
J.~Lin, C.~Song, K.~He, L.~Wang, and J.~E. Hopcroft, ``Nesterov accelerated
  gradient and scale invariance for adversarial attacks,'' in \emph{{ICLR}},
  2020.

\bibitem{long2022frequency}
Y.~Long, Q.~Zhang, B.~Zeng, L.~Gao, X.~Liu, J.~Zhang, and J.~Song, ``Frequency
  domain model augmentation for adversarial attack,'' in \emph{{ECCV}}, 2022.

\bibitem{wang2023structure}
X.~Wang, Z.~Zhang, and J.~Zhang, ``Structure invariant transformation for
  better adversarial transferability,'' in \emph{{ICCV}}, 2023.

\bibitem{wang2024boosting}
K.~Wang, X.~He, W.~Wang, and X.~Wang, ``Boosting adversarial transferability by
  block shuffle and rotation,'' in \emph{{CVPR}}, 2024.

\bibitem{zhou2025leveraging}
Z.~Zhou, L.~Li, Y.~Ren, C.~Qin, and G.~Feng, ``Leveraging spatial invariance to
  boost adversarial transferability,'' in \emph{{ICCV}}, 2025.

\bibitem{SzegedyZSBEGF13}
C.~Szegedy, W.~Zaremba, I.~Sutskever, J.~Bruna, D.~Erhan, I.~J. Goodfellow, and
  R.~Fergus, ``Intriguing properties of neural networks,'' in \emph{{ICLR}},
  2014.

\bibitem{zhao2021success}
Z.~Zhao, Z.~Liu, and M.~Larson, ``On success and simplicity: A second look at
  transferable targeted attacks,'' in \emph{{NeurIPS}}, 2021.

\bibitem{liang2024improving}
K.~Liang, X.~Dai, Y.~Li, D.~Wang, and B.~Xiao, ``Improving transferable
  targeted attacks with feature tuning mixup,'' in \emph{{CVPR}}, 2025.

\bibitem{poursaeed2018generative}
O.~Poursaeed, I.~Katsman, B.~Gao, and S.~Belongie, ``Generative adversarial
  perturbations,'' in \emph{{CVPR}}, 2018.

\bibitem{naseer2019cross}
M.~M. Naseer, S.~H. Khan, M.~H. Khan, F.~Shahbaz~Khan, and F.~Porikli,
  ``Cross-domain transferability of adversarial perturbations,'' in
  \emph{{NeurIPS}}, 2019.

\bibitem{li2025uvattack}
Y.~Li, K.~Liang, and B.~Xiao, ``{UV}-attack: Physical-world adversarial attacks
  on person detection via dynamic-{N}e{RF}-based {UV} mapping,'' in
  \emph{{ICLR}}, 2025.

\bibitem{ilyas2018black}
A.~Ilyas, L.~Engstrom, A.~Athalye, and J.~Lin, ``Black-box adversarial attacks
  with limited queries and information,'' in \emph{{ICML}}, 2018.

\bibitem{Wu0X0M20}
D.~Wu, Y.~Wang, S.~Xia, J.~Bailey, and X.~Ma, ``Skip connections matter: On the
  transferability of adversarial examples generated with resnets,'' in
  \emph{{ICLR}}, 2020.

\bibitem{GuoLC20}
Y.~Guo, Q.~Li, and H.~Chen, ``Backpropagating linearly improves transferability
  of adversarial examples,'' in \emph{NeurIPS}, 2020.

\bibitem{liang2023styless}
K.~Liang and B.~Xiao, ``Styless: Boosting the transferability of adversarial
  examples,'' in \emph{{CVPR}}, 2023.

\bibitem{HuangKGHBL19}
Q.~Huang, I.~Katsman, Z.~Gu, H.~He, S.~J. Belongie, and S.~Lim, ``Enhancing
  adversarial example transferability with an intermediate level attack,'' in
  \emph{{ICCV}}, 2019.

\bibitem{dai2024advdiff}
X.~Dai, K.~Liang, and B.~Xiao, ``Advdiff: Generating unrestricted adversarial
  examples using diffusion models,'' in \emph{{ECCV}}, 2024.

\bibitem{DongPSZ19}
Y.~Dong, T.~Pang, H.~Su, and J.~Zhu, ``Evading defenses to transferable
  adversarial examples by translation-invariant attacks,'' in \emph{{CVPR}},
  2019.

\bibitem{lin2024boosting}
Q.~Lin, C.~Luo, Z.~Niu, X.~He, W.~Xie, Y.~Hou, L.~Shen, and S.~Song, ``Boosting
  adversarial transferability across model genus by deformation-constrained
  warping,'' in \emph{{AAAI}}, 2024.

\bibitem{zhu2024learning}
R.~Zhu, Z.~Zhang, S.~Liang, Z.~Liu, and C.~Xu, ``Learning to transform
  dynamically for better adversarial transferability,'' in \emph{{CVPR}}, 2024.

\bibitem{guo2025boosting}
Y.~Guo, W.~Liu, Q.~Xu, S.~Zheng, S.~Huang, Y.~Zang, S.~Shen, C.~Wen, and
  C.~Wang, ``Boosting adversarial transferability through augmentation in
  hypothesis space,'' in \emph{{CVPR}}, 2025.

\bibitem{LiuCLS17}
Y.~Liu, X.~Chen, C.~Liu, and D.~Song, ``Delving into transferable adversarial
  examples and black-box attacks,'' in \emph{{ICLR}}, 2017.

\bibitem{li2020learning}
Y.~Li, S.~Bai, Y.~Zhou, C.~Xie, Z.~Zhang, and A.~Yuille, ``Learning
  transferable adversarial examples via ghost networks,'' in \emph{{AAAI}},
  2020.

\bibitem{ZhouHCTHGY18}
W.~Zhou, X.~Hou, Y.~Chen, M.~Tang, X.~Huang, X.~Gan, and Y.~Yang,
  ``Transferable adversarial perturbations,'' in \emph{{ECCV}}, 2018.

\bibitem{Wang_2021_ICCV}
Z.~Wang, H.~Guo, Z.~Zhang, W.~Liu, Z.~Qin, and K.~Ren, ``Feature
  importance-aware transferable adversarial attacks,'' in \emph{{ICCV}}, 2021.

\bibitem{zhang2022improving}
J.~Zhang, W.~Wu, J.-t. Huang, Y.~Huang, W.~Wang, Y.~Su, and M.~R. Lyu,
  ``Improving adversarial transferability via neuron attribution-based
  attacks,'' in \emph{{CVPR}}, 2022.

\bibitem{liao2018defense}
F.~Liao, M.~Liang, Y.~Dong, T.~Pang, X.~Hu, and J.~Zhu, ``Defense against
  adversarial attacks using high-level representation guided denoiser,'' in
  \emph{{CVPR}}, 2018.

\bibitem{tramer2017ensemble}
F.~Tram{\`e}r, A.~Kurakin, N.~Papernot, I.~Goodfellow, D.~Boneh, and
  P.~McDaniel, ``Ensemble adversarial training: Attacks and defenses,'' in
  \emph{{ICLR}}, 2018.

\bibitem{wong2020fast}
E.~Wong, L.~Rice, and J.~Z. Kolter, ``Fast is better than free: Revisiting
  adversarial training,'' in \emph{{ICLR}}, 2020.

\bibitem{cohen2019certified}
J.~Cohen, E.~Rosenfeld, and Z.~Kolter, ``Certified adversarial robustness via
  randomized smoothing,'' in \emph{{ICML}}, 2019.

\bibitem{naseer2020self}
M.~Naseer, S.~Khan, M.~Hayat, F.~S. Khan, and F.~Porikli, ``A self-supervised
  approach for adversarial robustness,'' in \emph{{CVPR}}, 2020.

\bibitem{nie2022DiffPure}
W.~Nie, B.~Guo, Y.~Huang, C.~Xiao, A.~Vahdat, and A.~Anandkumar, ``Diffusion
  models for adversarial purification,'' in \emph{{ICML}}, 2022.

\bibitem{szeliski2022computer}
R.~Szeliski, \emph{Computer vision: algorithms and applications}.\hskip 1em
  plus 0.5em minus 0.4em\relax Springer Nature, 2022.

\bibitem{russakovsky2015imagenet}
O.~Russakovsky, J.~Deng, H.~Su, J.~Krause, S.~Satheesh, S.~Ma, Z.~Huang,
  A.~Karpathy, A.~Khosla, M.~Bernstein \emph{et~al.}, ``Imagenet large scale
  visual recognition challenge,'' \emph{International journal of computer
  vision ({IJCV})}, vol. 115, pp. 211--252, 2015.

\bibitem{krizhevsky2009learning}
A.~Krizhevsky, G.~Hinton \emph{et~al.}, ``Learning multiple layers of features
  from tiny images,'' 2009.

\bibitem{huang2017densely}
G.~Huang, Z.~Liu, L.~Van Der~Maaten, and K.~Q. Weinberger, ``Densely connected
  convolutional networks,'' in \emph{{CVPR}}, 2017.

\bibitem{sandler2018mobilenetv2}
M.~Sandler, A.~Howard, M.~Zhu, A.~Zhmoginov, and L.-C. Chen, ``Mobilenetv2:
  Inverted residuals and linear bottlenecks,'' in \emph{{CVPR}}, 2018.

\bibitem{dosovitskiy2020vit}
A.~Dosovitskiy, L.~Beyer, A.~Kolesnikov, D.~Weissenborn, X.~Zhai,
  T.~Unterthiner, M.~Dehghani, M.~Minderer, G.~Heigold, S.~Gelly, J.~Uszkoreit,
  and N.~Houlsby, ``An image is worth 16x16 words: Transformers for image
  recognition at scale,'' \emph{ICLR}, 2021.

\bibitem{graham2021levit}
B.~Graham, A.~El{-}Nouby, H.~Touvron, P.~Stock, A.~Joulin, H.~J{\'{e}}gou, and
  M.~Douze, ``Levit: a vision transformer in convnet's clothing for faster
  inference,'' in \emph{{ICCV}}, 2021.

\bibitem{d2021convit}
S.~d'Ascoli, H.~Touvron, M.~L. Leavitt, A.~S. Morcos, G.~Biroli, and L.~Sagun,
  ``Convit: Improving vision transformers with soft convolutional inductive
  biases,'' in \emph{{ICML}}, 2021.

\bibitem{heo2021rethinking}
B.~Heo, S.~Yun, D.~Han, S.~Chun, J.~Choe, and S.~J. Oh, ``Rethinking spatial
  dimensions of vision transformers,'' in \emph{{ICCV}}, 2021.

\bibitem{kariyappa2019improving}
S.~Kariyappa and M.~K. Qureshi, ``Improving adversarial robustness of ensembles
  with diversity training,'' \emph{arXiv preprint arXiv:1901.09981}, 2019.

\bibitem{yang2020dverge}
H.~Yang, J.~Zhang, H.~Dong, N.~Inkawhich, A.~Gardner, A.~Touchet, W.~Wilkes,
  H.~Berry, and H.~Li, ``Dverge: diversifying vulnerabilities for enhanced
  robust generation of ensembles,'' in \emph{{NeurIPS}}, 2020.

\end{thebibliography}
 
\end{document}